\documentclass{article} 
\usepackage{iclr2026_conference,times}

\usepackage{amsmath,amsfonts,bm}

\def\eqref#1{equation~\ref{#1}}

\def\1{\bm{1}}

\DeclareMathAlphabet{\mathsfit}{\encodingdefault}{\sfdefault}{m}{sl}
\SetMathAlphabet{\mathsfit}{bold}{\encodingdefault}{\sfdefault}{bx}{n}

\usepackage{hyperref}
\usepackage{url}

\usepackage[capitalize]{cleveref}
\crefname{equation}{}{}
\usepackage{hyperref}       
\usepackage{amsmath}
\usepackage{amssymb}
\usepackage{mathtools}
\usepackage{amsthm}
\usepackage{bm}
\usepackage{booktabs}
\usepackage{multirow, makecell}
\usepackage{multicol}
\usepackage{subcaption}
\usepackage{threeparttable} 
\usepackage{rotating}
\usepackage{algorithm}
\usepackage{setspace}
\usepackage[noend]{algpseudocode}

\newtheorem{theorem}{Theorem}[section]
\newtheorem{proposition}[theorem]{Proposition}

\title{\moa: Faster and Exact Attention-aware Quantization without Backpropagation}

\iclrfinalcopy
\author{Junhan Kim, Yeo Jeong Park, Seungwoo Son, Chungman Lee, \\ \textbf{Ho-young Kim}, \textbf{Joonyoung Kim}, \textbf{Yongkweon Jeon}\thanks{Corresponding Author} \\
Samsung Research, Seoul, Korea \\
\footnotesize{\texttt{junhankim@islab.snu.ac.kr, \{yeo\_j.park, dragwon.jeon\}@samsung.com}}
}

\newcommand{\eg}{\emph{e.g.}}
\newcommand{\ie}{\emph{i.e.}}

\newcommand{\boa}{\textsc{BoA}}
\newcommand{\moa}{\textsc{TurboBoA}}
\newcommand{\trace}{\text{tr}}

\newcommand{\diag}{\operatornamewithlimits{diag}}

\newcommand{\chol}{\operatornamewithlimits{Chol}}

\usepackage[table]{xcolor}
\usepackage{xcolor}

\begin{document}

\maketitle

\begin{abstract}
    The rapid growth of large language models (LLMs) has heightened the importance of post-training quantization (PTQ) for reducing memory and computation costs.
    Among PTQ methods, GPTQ has gained significant attention for its efficiency, enabling billion-scale LLMs to be quantized within a few GPU hours. 
    However, GPTQ's assumption of layer-wise independence leads to severe accuracy drops in low-bit regimes. 
    Recently, \boa \ improved upon GPTQ by incorporating inter-layer dependencies within attention modules, but its reliance on sequential quantization across all out-channels makes it substantially less efficient.
    In this paper, we propose \moa, a new backpropagation-free PTQ algorithm that preserves the accuracy benefits of \boa \ while significantly accelerating the process. 
    The proposed \moa \ introduces three key innovations: (i) joint quantization of multiple out-channels with a closed-form error compensation rule, which reduces sequential bottlenecks and yields more than a three-fold speedup; (ii) a correction mechanism for errors propagated from preceding quantized layers; and (iii) adaptive grid computation with coordinate descent refinement to maintain alignment during iterative updates.
    Extensive experiments demonstrate that \moa \ delivers substantial acceleration over \boa \ while consistently improving accuracy.
    When combined with outlier suppression techniques, it achieves state-of-the-art results in both weight-only and weight-activation quantization.
    The code will be available at \url{https://github.com/SamsungLabs/TurboBoA}.
\end{abstract}

\section{Introduction}  \label{sec:intro}

The rapid scaling of large language models (LLMs)~\citep{touvron2023llama, touvron2023llama2} has dramatically increased their memory footprint and computational requirements, making deployment on resource-constrained hardware challenging.
As a practical solution to reduce memory usage and accelerate inference, post-training quantization (PTQ), which reduces the precision of weights and activations using only a small calibration dataset, has received considerable attention.

The PTQ pipeline for LLMs typically involves two major stages.
First, the model is transformed to be more robust to quantization by suppressing outliers in weights and activations through scaling (\eg, SmoothQuant~\citep{xiao2023smoothquant}) or rotation (\eg, QuaRot~\citep{ashkboos2024quarot}).
Next, the transformed model is quantized under specific bit-width constraints.
For weight quantization, backpropagation-free methods exploiting Hessian-guided error compensation have been widely adopted~\citep{frantar2023optq, kim2024boa, li2025gptaq}, as they facilitate efficient optimization of quantized weights without gradient-based training.

Among backpropagation-free methods, GPTQ is a representative approach known for its efficiency, enabling the quantization of billion-scale LLMs within a few GPU hours~\citep{frantar2023optq}.
However, GPTQ assumes layer-wise independence, which leads to severe accuracy degradation in low-bit regimes (\eg, INT2). 
Recently, \boa \ addressed this by exploiting attention reconstruction errors in the Hessian approximation~\citep{kim2024boa}.
By capturing cross-layer dependencies within attention modules, \boa \ yields substantial accuracy gains over GPTQ.
However, \boa \ introduces a significant computational bottleneck: it performs quantization \emph{sequentially} across out-channels to compensate for the quantization error of each out-channel (see \cref{fig:quantization order}). 
Such sequential process, although necessary for precise error compensation, severely slows down the overall process and makes \boa \ substantially less efficient than GPTQ.

The primary goal of this paper is to accelerate \boa \ without sacrificing accuracy and even to achieve further performance improvements.
Our main contributions are as follows:
\begin{itemize}
    \item We propose \moa, which significantly accelerates \boa \ (\textbf{\cref{subsec:acceleration}}). 
    Our key idea is to quantize multiple out-channels simultaneously, thereby reducing the number of sequential operations while explicitly incorporating their dependencies into the error compensation (\textbf{Proposition~\ref{prop:correction rule_multiple out-channels}}).
    Our timing measurements demonstrate that the proposed joint quantization leads to more than a three-fold speedup over \boa \ (\textbf{\cref{tab:ablation_multiple row processing}}).
    
    \item We incorporate two features into \moa \ to enhance its performance (\textbf{Sections~\ref{subsec:residual compensation} and~\ref{subsec:enhance_grid}}). 
    First, \moa \ compensates for errors propagated from preceding quantized layers, mitigating error accumulation across layer depths (\textbf{Proposition~\ref{prop:correction rule_multiple out-channels and residual}}).
    Second, \moa \ adaptively determines quantization grids to align them with weights iteratively updated for the error compensation and further refines grids to reduce attention reconstruction errors (\textbf{Proposition~\ref{prop:cd-based scale optimization}}).

    \item From extensive experiments, we demonstrate that \moa \ delivers substantial acceleration over \boa \ while achieving superior accuracy (\textbf{\cref{tab:ablation_performance improvement}}).
    Furthermore, when integrated with outlier suppression techniques, \moa \ achieves state-of-the-art results for both weight-only and weight-activation quantization (\textbf{Tables~\ref{tab:weight_only_quant_with_transform_main} and~\ref{tab:weight_act_quant_with_transform_main}}).
\end{itemize}

\paragraph{Notations} 
We use lowercase letters to denote vectors (\eg, $\mathbf{w}$) and uppercase letters for matrices (\eg, $\mathbf{W}$).
$w_{i}$ denotes the $i$-th element in $\mathbf{w}$, and $W_{i, j}$ is the $(i, j)$-th entry in $\mathbf{W}$.
We denote the $i$-th row of $\mathbf{W}$, which corresponds to the $i$-th out-channel, by $\mathbf{W}_{i, :}$ and the $j$-th column of $\mathbf{W}$ by $\mathbf{W}_{:, j}$.
The submatrix of $\mathbf{W}$ consisting of the rows indexed by the index set $B$ is denoted by $\mathbf{W}_{B, :}$.
Similarly, $\mathbf{W}_{:, B}$ denotes the submatrix of $\mathbf{W}$ with  the columns indexed by $B$.
$\mathbf{e}_{i}$ is the vector with a 1 in the $i$-th coordinate and 0's elsewhere, and $\mathbf{I}$ denotes the identity matrix.
$\mathbf{0}_{d_{1} \times d_{2}}$ and $\mathbf{1}_{d_{1} \times d_{2}}$ are $(d_{1} \times d_{2})$-dimensional matrices with entries being zeros and ones, respectively.

\section{Related Works}  \label{sec:background}

\subsection{LLM Quantization}

The main goal of PTQ is to minimize the degradation in task loss induced by quantization, which can be relaxed to the layer-wise reconstruction problem~\citep{lecun1989optimal, nagel2020up}
\begin{align}
    \min_{\mathbf{Q} \in \mathcal{Q}}~& \left \| \left ( \mathbf{Q} - \mathbf{W} \right ) \mathbf{X} \right \|_{F}^{2}, \label{eq:layer-wise reconstruction}
\end{align}
where $\mathbf{W} \in \mathbb{R}^{d_{out} \times d_{in}}$ is a weight matrix for one layer, $\mathbf{X} \in \mathbb{R}^{d_{in} \times L}$ is its input of length $L$, and $\mathcal{Q}$ is the set of discrete quantized weights $\mathbf{Q}$.
If channel-wise quantization is adopted, $\mathbf{Q}$ can be expressed as 
\begin{align}
    \mathbf{Q} = \diag (\mathbf{s} ) \mathbf{W}_{int}, \mathbf{W}_{int} \in \{ 0, \ldots, 2^{b}-1 \}^{d_{out} \times d_{in}}
\end{align}
where $\mathbf{s} \in \mathbb{R}^{d_{out}}$ is a scale vector and $b$ is the target bit-width.

Early PTQ approaches aimed to reformulate the assignment of discrete quantized values into a continuous optimization problem, enabling quantized weights to be learned through gradient-based training. 
Representative algorithms include AdaRound, which introduced differentiable approximations for the rounding operation~\citep{nagel2020up}, and BRECQ, which further extended this idea to the block-wise reconstruction problem to consider cross-layer dependencies~\citep{li2021brecq}. 
Although these methods have been successful for small-scale models (\eg, ResNet), they depend on time-consuming gradient-based training, which renders them impractical for LLMs with billions of parameters.

Recent research has therefore focused on developing cost-effective alternatives for LLM quantization~\citep{frantar2023optq, jeon2023frustratingly, aespa, kim2024boa}. 
These works can be categorized into two orthogonal classes. 
The first is backpropagation-free methods, which resort to Hessian-guided error compensation (\eg, GPTQ~\citep{frantar2023optq} and \boa~\citep{kim2024boa}). 
The second is transformation-based methods, which suppress outliers via scaling or rotation, thereby transforming LLMs into a more quantization-friendly form (\eg, SmoothQuant~\citep{xiao2023smoothquant} and QuaRot~\citep{ashkboos2024quarot}).

Our approach belongs to the backpropagation-free class and further improves \boa \ by enhancing both efficiency and accuracy. 
Furthermore, similar to GPTQ and \boa, our method can be effectively combined with transformation-based methods, demonstrating strong complementarity between the two classes.

\subsection{Backpropagation-free Weight Quantization}

Backpropagation-free PTQ algorithms, which rely on the Hessian-guided error compensation, have been widely adopted for efficient LLM quantization~\citep{frantar2023optq, li2025gptaq, kim2024boa}. 
These algorithms rapidly quantize LLMs by iteratively conducting quantization and error correction, which is given as~\citep{frantar2023optq}
\begin{align} \label{eq:weight-update}
    \Delta \mathbf{w} = - \frac{w_{p} - q_{p}}{U_{p, p}} \mathbf{U}_{p, :} \text{ where } \mathbf{U}
        &= \chol ( \mathbf{H}^{-1} )^{T},
\end{align}
where $q_{p}$ is the quantized version of the weight $w_{p}$, $\mathbf{H}$ is the Hessian matrix, and $\chol(\cdot)$ denotes a Cholesky decomposition, that is, $\mathbf{U}$ is upper triangular such that $\mathbf{H}^{-1} = \mathbf{U}^{T} \mathbf{U}$.

The first algorithm that successfully scaled this principle to LLMs was GPTQ~\citep{frantar2023optq}. 
However, GPTQ approximates the Hessian based on layer-wise reconstruction errors, failing to account for inter-layer dependencies and resulting in suboptimal performance, particularly at low bit-widths (\eg, INT2). 
Recently, \boa \ addressed this issue by exploiting attention reconstruction errors in the Hessian approximation~\citep{kim2024boa}. 
The resulting Hessians explicitly model dependencies between out-channels (see $\mathbf{H}_{out}$ in~\cref{tab:hessians}), enabling the error compensation for each out-channel and yielding substantial accuracy gains over GPTQ.
Nevertheless, such improved Hessians make sequential processing across out-channels unavoidable; the second out-channel can be quantized after compensating for the quantization error induced by the first out-channel, which differs from GPTQ that quantizes all out-channels simultaneously (see \cref{fig:quantization order}).
To alleviate this bottleneck, \boa \ parallelizes quantization across different attention heads (\eg, quantizing the first out-channel of all heads concurrently) by assuming head-wise independence. 
While this strategy provides some acceleration, \boa \ still remains substantially more time-consuming than GPTQ, highlighting a trade-off between accuracy and efficiency.

\begin{table*}[!htb]
    \renewcommand{\arraystretch}{1.0}
    \footnotesize
    \centering
    \caption{Loss used to approximate Hessians and the corresponding Hessians in GPTQ and \boa.}
    \vspace{-0.25cm}
    \begin{threeparttable}
    \begin{tabular}{c c c c}
        \toprule
        Method & Layer & Loss $(\| \mathbf{G} \Delta \hspace{-.57mm} \mathbf{W} \mathbf{X} \|_{F}^{2})$ & $\mathbf{H} = \mathbf{H}_{in} \otimes \mathbf{H}_{out}$ \\
        \toprule
        GPTQ
        & $\mathbf{W}_{\{ Q, K, V \}}$ & $\| \Delta \hspace{-.57mm} \mathbf{W} \mathbf{X} \|_{F}^{2}$ & $\mathbf{X} \mathbf{X}^{T} \otimes \mathbf{I}$ \\
        \midrule
        \multirowcell{3}{\boa}
        & $\mathbf{W}_{Q, h}$ & $\| \mathbf{K}_{h} \Delta \hspace{-.57mm} \mathbf{W}_{Q, h} \mathbf{X} \|_{F}^{2}$ & $\mathbf{X} \mathbf{X}^{T} \otimes \mathbf{K}_{h}^{T} \mathbf{K}_{h}$ \\
        & $\mathbf{W}_{K, h}$ & $\| \mathbf{Q}_{h} \Delta \hspace{-.57mm} \mathbf{W}_{K, h} \mathbf{X} \|_{F}^{2}$ & $\mathbf{X} \mathbf{X}^{T} \otimes \mathbf{Q}_{h}^{T} \mathbf{Q}_{h}$ \\
        & $\mathbf{W}_{V, h}$ & $\| \mathbf{W}_{out, h} \Delta \hspace{-.57mm} \mathbf{W}_{V, h} \mathbf{X} \mathbf{A}_{h}^{T} \|_{F}^{2}$ & $\mathbf{X} \mathbf{A}_{h}^{T} \mathbf{A}_{h} \mathbf{X}^{T} \otimes \mathbf{W}_{\text{out}, h}^{T} \mathbf{W}_{\text{out}, h}$ \\
        \bottomrule
    \end{tabular}
    \begin{tablenotes}
        \item[*] $h$ denotes the index of the attention head.
    \end{tablenotes}
    \end{threeparttable}
    \label{tab:hessians}
    \vspace{-.25cm}
\end{table*}

\section{\moa}

We now introduce the proposed \moa.
To enhance both the efficiency and accuracy of \boa, we introduce three key innovations, each of which will be described in the following subsections in detail. 

\begin{figure*}[t]
  \begin{subfigure}{.49\linewidth}
    \centering
    \includegraphics[width=.9\textwidth, keepaspectratio]{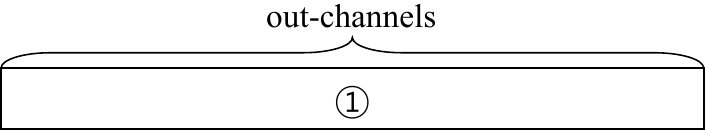}
    \vspace{-.1cm}
    \caption{GPTQ}
  \end{subfigure}
  \begin{subfigure}{.49\linewidth}
    \centering
    \includegraphics[width=.9\textwidth, keepaspectratio]{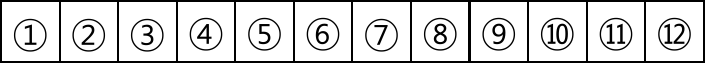}
    \vspace{-.1cm}
    \caption{\boa}
  \end{subfigure}
  
  \vspace{.2cm}
  
  \begin{subfigure}{1\linewidth}
    \centering
    \includegraphics[width=.441\textwidth, keepaspectratio]{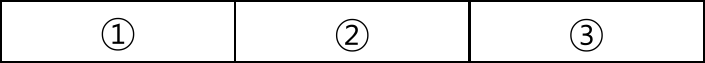}
    \vspace{-.1cm}
    \caption{\moa \ ($N=4$)}
  \end{subfigure}
  \vspace{-0.5cm}
  \caption{Quantization orders in GPTQ, \boa, and the proposed \moa. (a) GPTQ quantizes all out-channels jointly but without error correction. (b) \boa \ compensates for the quantization error but requires fully sequential processing across out-channels. (c) \moa \ reduces sequential operations by quantizing multiple $N$ out-channels jointly while still applying error compensation.}
  \label{fig:quantization order}
\end{figure*}

\subsection{Simultaneous Quantization of Multiple Out-channels} \label{subsec:acceleration}

As described earlier, \boa \ sequentially quantizes out-channels one by one.
This means that when quantizing a weight matrix with 128 out-channels (\eg, query, key, and value projection weights in Llama3-8B), BoA requires 128 sequential operations.
Consequently, BoA is substantially more time-consuming than GPTQ, in which all out-channels are quantized in parallel.

To accelerate the quantization process, \moa \ quantizes multiple $N$ out-channels simultaneously, thereby reducing the number of sequential operations (see \cref{fig:quantization order}).
 In the previous example, when $N=16$, the number of sequential operations decreases from 128 to 8.
We note that while $N$ out-channels are quantized together as if they were mutually independent (as in GPTQ), we explicitly incorporate their dependencies into the error compensation step.
To do so, instead of na\"{i}vely adding weight compensation for each out-channel, we formulate the problem of compensating for the errors of multiple out-channels as
\begin{align} \label{eq:error correction problem_multiple weights}
    \begin{split}
    \min_{\Delta \hspace{-.57mm} \mathbf{W}}~~&\| \mathbf{G} \Delta \hspace{-.57mm} \mathbf{W} \mathbf{X} \|_{F}^{2}, \\
    \text{s.t. }~~&\mathbf{e}_{i}^{T} \Delta \hspace{-.57mm} \mathbf{W} = \mathbf{Q}_{i, :} \hspace{-.57mm} - \hspace{-.57mm}  \mathbf{W}_{i, :}~(0 \le i < N),
    \end{split}
\end{align}
where $\mathbf{Q}_{i, :}$ is the quantized version of $\mathbf{W}_{i, :}$ and we use the unified notation $\| \mathbf{G} \Delta \hspace{-.57mm} \mathbf{W} \mathbf{X}\|_{F}^{2}$ to denote the attention reconstruction errors in \cref{tab:hessians} (\eg, $\mathbf{G} = \mathbf{K}_{h}$ for the query projection weight $\mathbf{W}_{Q, h}$).
In the following proposition, we present a closed-form solution to~\cref{eq:error correction problem_multiple weights}. 

\begin{proposition} \label{prop:correction rule_multiple out-channels}
    Let $\mathbf{W}$ be a matrix whose Hessian is given as $\mathbf{H} = \mathbf{H}_{in} \otimes \mathbf{H}_{out}$. 
    Suppose the first $N$ out-channels of $\mathbf{W}$ have been quantized simultaneously and the other out-channels are updated to minimize the attention reconstruction error in~\cref{eq:error correction problem_multiple weights}.
    Then, the update $[\Delta \hspace{-.57mm} \mathbf{W}]_{N:, :}$ satisfies
    \begin{align} \label{eq:update rule_multiple weights}
        [\Delta \hspace{-.57mm} \mathbf{W}]_{N:, :}
            &\hspace{-.57mm} = \hspace{-.57mm}-[\mathbf{U}_{out}^{T}]_{N:, B}[\mathbf{U}_{out}^{T}]_{B, B}^{-1} (\mathbf{W}_{B, :} \hspace{-.57mm} - \hspace{-.57mm}  \mathbf{Q}_{B, :}),
    \end{align}
    where $B = \{0, \ldots, N-1\}$ and $\mathbf{U}_{out} = \chol(\mathbf{H}_{out}^{-1})^{T}$.
\end{proposition}
\proof
See Appendix~\ref{appendix:proof_update_formula_multiple_weights}.
\endproof

A pertinent question is whether such joint quantization inevitably leads to accuracy degradation.
In the conventional \boa, the quantization error of the first out-channel can be compensated by all subsequent out-channels; \eg, 127 out-channels participate in the error compensation in the previous example.
Whereas, in our approach, multiple out-channels are quantized at once, so the number of out-channels available for the error correction decreases (\eg, 127→112 when $N=16$). 
In \cref{subsec:ablation}, we will empirically show that the degradation arising from the reduced error correction flexibility is negligible, even in the low-bit regime (see \cref{tab:ablation_multiple row processing}).

\begin{algorithm*}[t]
\begin{spacing}{1.05}
\caption{\moa} 
\footnotesize
\label{algo:turboboa}
\renewcommand\algorithmicrequire{\textbf{Input}:}
\renewcommand\algorithmicensure{\textbf{Output}:}
\begin{algorithmic}[1]
\Require weights $\mathbf{W}_{\{ Q, K, V \}} \in \mathbb{R}^{H \times d_{h} \times d}$, target bit-width $b$, inputs $\mathbf{X} \in \mathbb{R}^{d \times L}$, FP representation $\widetilde{\mathbf{X}} \in \mathbb{R}^{d \times L}$, number $N$ of out-channels quantized simultaneously, and stabilization coefficient $\alpha$
\Ensure quantized weights $\mathbf{Q}_{\{ Q, K, V \}}$
    \For{$\mathbf{W} \in \{ \mathbf{W}_{Q}, \mathbf{W}_{K}, \mathbf{W}_{V} \} $}
    \State \hspace{-3.5mm} Initialize quantized outputs and integer weights: $\mathbf{Q}_{h}, \mathbf{W}_{int, h} \leftarrow \mathbf{0}_{d_{h} \times d}$
    \State \hspace{-3.5mm} Initialize out-channel scales: $\mathbf{s}_{h} \leftarrow \mathbf{1}_{d_{h} \times 1}$
    \State \hspace{-3.5mm} Compute attention-aware Hessians $\mathbf{H}_{in, h}$ and $\mathbf{H}_{out, h}$
    \State \hspace{-3.5mm} Compute $\mathbf{U}_{in, h} = \chol(\mathbf{H}_{in, h}^{-1})^{T}$, $\mathbf{U}_{out, h} = \chol(\mathbf{H}_{out, h}^{-1})^{T}$, and $\mathbf{R} = \alpha (\mathbf{X} - \widetilde{\mathbf{X}}) \mathbf{X}^{T}$
    \State \hspace{-3.5mm} Initialize updated weights: $\widetilde{\mathbf{W}} \leftarrow \mathbf{W}$
    \State \hspace{-3.5mm} \textbf{for} $i=0, N, 2N, \ldots$ \textbf{do}
    \State Take $N$ out-channels to be quantized jointly: $\mathbf{W}^{(i)} \leftarrow [\widetilde{\mathbf{W}}_{h}]_{B, :}~(B = \{ i, i+1, \ldots, i+N-1 \})$
    \State Set scales: $[\mathbf{s}_{h}]_{B} \leftarrow \min_{\mathbf{s}} \trace ( \Delta \hspace{-.57mm} \mathbf{W}^{(i)} \mathbf{H}_{in, h} (\Delta \hspace{-.57mm} \mathbf{W}^{(i)})^{T} )$
    \State Quantize $\mathbf{W}^{(i)}$: $([\mathbf{Q}_{h}]_{B, :}, [\mathbf{W}_{int, h}]_{B, :}) \leftarrow \text{GPTAQ} (\mathbf{W}^{(i)}, \mathbf{U}_{in, h}, \mathbf{R}, [\mathbf{s}_{h}]_{B})$~~~~~~~~~~~$\triangleright$ see \cref{algo:gptaq}
    \State Update remaining out-channels: 
        \vspace{-1.5mm}
        \begin{align*}
            \begin{split}
            [\widetilde{\mathbf{W}}_{h}]_{i+N:, :} 
            &\leftarrow [\widetilde{\mathbf{W}}_{h}]_{i+N:, :} - [\mathbf{U}_{out, h}^{T}]_{i+N:, B}[\mathbf{U}_{out, h}^{T}]_{B, B}^{-1} ( [\widetilde{\mathbf{W}}_{h}]_{B, :} - [\mathbf{Q}_{h}]_{B, :} ) \\
            &~~~+[\mathbf{U}_{out, h}^{T}]_{i+N:, B}[\mathbf{U}_{out, h}^{T}]_{B, B}^{-1} [\widetilde{\mathbf{W}}_{h}]_{B, :} \mathbf{R} \mathbf{H}_{in, h}^{-1}
            \end{split}
        \end{align*}
    \State \hspace{-3.5mm} Refine scales: $\mathbf{s}_{h} \leftarrow \min_{\mathbf{s}} \| \mathbf{G}_{h} (\diag(\mathbf{s}) \mathbf{W}_{int, h} - \mathbf{W}_{h}) \mathbf{X} + \mathbf{G}_{h} \mathbf{W}_{h} \Delta \hspace{-.57mm} \mathbf{X} \|_{F}^{2}$~~~~~~~~~~~~~~$\triangleright$ see~\cref{algo:cd}
    \State \hspace{-3.5mm} Update quantized weights: $\mathbf{Q}_{h} \leftarrow \diag(\mathbf{s}_{h}) \mathbf{W}_{int, h}$
    \EndFor
\end{algorithmic}
\end{spacing}
\end{algorithm*}

\subsection{Error Compensation for Pre-quantized Layers}
\label{subsec:residual compensation}

Another limitation of \boa \ is that it does not account for quantization errors propagated from previously quantized layers.
During quantization, errors produced in one layer propagate to subsequent layers by perturbing their input distributions, and these deviations accumulate as the network depth increases, as reported in GPTAQ~\citep{li2025gptaq}.

Let $\widetilde{\mathbf{X}}$ be the original full-precision (FP) input.
We observe that the input deviation $\Delta \hspace{-.57mm} \mathbf{X} := \mathbf{X} - \widetilde{\mathbf{X}}$ introduces additional distortion in the attention output as follows:
\begin{align}
    \mathbf{G} \mathbf{Q} \mathbf{X} - \mathbf{G} \mathbf{W} \widetilde{\mathbf{X}}
        &= \mathbf{G} ( \mathbf{Q} - \mathbf{W} )\mathbf{X} + \mathbf{G} \mathbf{W} (\mathbf{X} - \widetilde{\mathbf{X}}) 
        = \mathbf{G} \Delta \hspace{-.57mm} \mathbf{W} \mathbf{X} + \mathbf{G} \mathbf{W} \Delta \hspace{-.57mm} \mathbf{X}. \label{eq:attention output deviation_res}
\end{align}
Here, $\mathbf{G} \Delta \hspace{-.57mm} \mathbf{W} \mathbf{X}$ corresponds to the error introduced by quantizing the current layer, while $\mathbf{G} \mathbf{W} \Delta \hspace{-.57mm} \mathbf{X}$ captures the output deviation induced by the perturbed input. 
We explicitly incorporate the additional distortion $\mathbf{G} \mathbf{W} \Delta \hspace{-.57mm} \mathbf{X}$ into the error compensation.
Specifically, after quantizing $N$ out-channels $\mathbf{W}_{B, :}$, we compensate for both the error introduced by the weight perturbation (\ie, $\mathbf{G}_{:, B} \Delta \hspace{-.57mm} \mathbf{W}_{B, :} \mathbf{X}$) and the error incurred by the input deviation $\Delta \hspace{-.57mm} \mathbf{X}$ (\ie, $\mathbf{G}_{:, B} \mathbf{W}_{B, :} \Delta \hspace{-.57mm} \mathbf{X}$), which reformulates the error compensation problem in~\cref{eq:error correction problem_multiple weights} as
\begin{align} \label{eq:error correction problem_multiple weights and residual}
    \begin{split}
    \min_{\Delta \hspace{-.57mm} \mathbf{W}}~~&\| \mathbf{G} \Delta \hspace{-.57mm} \mathbf{W} \mathbf{X} + \mathbf{G}_{:, B} \mathbf{W}_{B, :} \Delta \hspace{-.57mm} \mathbf{X} \|_{F}^{2}, \\
    \text{s.t. }~~&\mathbf{e}_{i}^{T} \Delta \hspace{-.57mm} \mathbf{W} = \mathbf{Q}_{i, :} \hspace{-.57mm} - \hspace{-.57mm} \mathbf{W}_{i, :}~(0 \le i < N).
    \end{split}
\end{align}
The following proposition provides a closed-form solution to the above problem. 

\begin{proposition} \label{prop:correction rule_multiple out-channels and residual}
    Let $\mathbf{W}$ be a matrix whose Hessian is given as $\mathbf{H} = \mathbf{H}_{in} \otimes \mathbf{H}_{out}$. 
    Suppose the first $N$ out-channels of $\mathbf{W}$ have been quantized simultaneously, where the input $\mathbf{X}$ is distorted from the FP representation $\widetilde{\mathbf{X}}$ due to the quantization errors produced in earlier layers.
    Then, the update $[\Delta \hspace{-.57mm} \mathbf{W}]_{N:, :}$ of the other out-channels to minimize the attention reconstruction error in~\cref{eq:error correction problem_multiple weights and residual} is
    \begin{align} \label{eq:update rule_multiple weights and residual}
        [\Delta \hspace{-.57mm} \mathbf{W}]_{N:, :}
            &\hspace{-.57mm}=\hspace{-.57mm} -[\mathbf{U}_{out}^{T}]_{N:, B} [\mathbf{U}_{out}^{T}]_{B, B}^{-1} \left ( (\mathbf{W}_{B, :} \hspace{-.57mm} - \hspace{-.57mm}  \mathbf{Q}_{B, :}) \hspace{-.57mm} - \hspace{-.57mm} \mathbf{W}_{B, :} \mathbf{R} \mathbf{H}_{in}^{-1} \hspace{-.57mm} \right ),
    \end{align}
    where $B = \{0, \ldots, N-1 \}$, $\mathbf{U}_{out} = \chol(\mathbf{H}_{out}^{-1})^{T}$, and $\mathbf{R} = \Delta \hspace{-.57mm} \mathbf{X} \mathbf{X}^{T}$.
\end{proposition}
\proof
See Appendix~\ref{appendix:proof_update_formula_residual}.
\endproof

Compared to the update rule in~\cref{eq:update rule_multiple weights} for the first layer, the update rule in~\cref{eq:update rule_multiple weights and residual} involves the additional term related to the input deviation $\Delta \hspace{-.57mm} \mathbf{X}$, which explicitly compensates for errors propagated across quantized layers and ensures the quantized model to replicate the behavior of the FP model more faithfully across all layers.

We note that our approach of incorporating the input deviation $\Delta \hspace{-.57mm} \mathbf{X}$ is motivated by the error compensation framework of GPTAQ~\citep{li2025gptaq}.
However, a key technical distinction lies in the structure of the Hessian $\mathbf{H}_{out}$.
While GPTAQ assumes $\mathbf{H}_{out} = \mathbf{I}$, which decouples out-channels and simplifies the optimization into a set of independent vector equations, our framework addresses a general (potentially dense) $\mathbf{H}_{out}$ to incorporate dependencies within out-channels.
This transition from a separable vector-wise optimization to a coupled matrix-wise formulation requires a more sophisticated derivation of the closed-form update rule, as detailed in Appendix~\ref{appendix:proof_update_formula_residual}.

\subsection{Adaptive Grid Selection with Attention-wise Refinement} \label{subsec:enhance_grid}

A remaining limitation of \boa \ lies in its grid computation. 
First, once initialized, the quantization grid remains fixed throughout the iterative process~\citep{kim2024boa}. 
However, since out-channels are continuously updated due to the error compensation, the initial grid becomes misaligned with the updated weights. 
This misalignment would particularly be severe in low-bit regimes, where large weight perturbations $(\mathbf{W}_{B, :} - \mathbf{Q}_{B, :})$ result in large updates $[\Delta \hspace{-.57mm} \mathbf{W}]_{N:, :}$ (see~\cref{eq:update rule_multiple weights}). 
Second, \boa \ computes the grid in a way to minimize the layer-wise reconstruction loss, which is not aligned with the goal of minimizing the attention reconstruction error. 

To address these issues, \moa \ determines the quantization grid immediately before each out-channel is quantized (line~9 in~\cref{algo:turboboa}), which ensures that every quantization step uses a grid aligned to the previously updated weights. 
To reduce unnecessary overhead, the grid is computed exclusively for the out-channels to be quantized at each quantization step.
Furthermore, we introduce a grid refinement step (line~12 in~\cref{algo:turboboa}). 
At this stage, we \emph{freeze} the integer weights $\mathbf{W}_{int} \in \{ 0, 1, \ldots, 2^{b}-1 \}^{d_{out} \times d_{in}}$ assigned through the iterative process (lines~7-11 in~\cref{algo:turboboa}) and refine only scales $\mathbf{s} \in \mathbb{R}^{d_{out}}$ to further reduce the attention reconstruction error in~\cref{eq:attention output deviation_res}:
\begin{align} \label{eq:attention-wise recon loss_scale optimize}
    \min_{\mathbf{s}}~\| \mathbf{G} (\diag(\mathbf{s}) \mathbf{W}_{int} - \mathbf{W}) \mathbf{X} + \mathbf{G} \mathbf{W} \Delta \hspace{-.57mm} \mathbf{X} \|_{F}^{2}.
\end{align}
To solve this problem, we adopt coordinate descent (CD), which iteratively updates one scale at a time while fixing the others. 
The following proposition presents the closed-form update rule for each CD step, which facilitates each scale update without a costly numerical optimization.

\begin{proposition} \label{prop:cd-based scale optimization}
    Let $\mathbf{W}$ be a matrix whose Hessian is given as $\mathbf{H} = \mathbf{H}_{in} \otimes \mathbf{H}_{out}$.
    Suppose $\mathbf{W}$ has been quantized to $\mathbf{Q} = \diag(\mathbf{s}) \mathbf{W}_{int}$ where $\mathbf{s}$ is the scale vector and $\mathbf{W}_{int}$ is the fixed integer weights.
    Suppose the scales $\mathbf{s}$ are refined to minimize the attention reconstruction error in~\cref{eq:attention-wise recon loss_scale optimize} via CD.
    Then, the update rule for each CD step is given as
    \begin{align*}
        s_{j}^{*}
            &= s_{j} + \frac{[\mathbf{W}_{int} (\mathbf{H}_{in} ( \mathbf{W} - \mathbf{Q} )^{T} - \mathbf{R}^{T} \mathbf{W}^{T} ) \mathbf{H}_{out}]_{j, j}}{[\mathbf{W}_{int} \mathbf{H}_{in} \mathbf{W}_{int}^{T}]_{j, j}[\mathbf{H}_{out}]_{j, j} },
    \end{align*}
where $\mathbf{R}=\Delta \hspace{-.57mm} \mathbf{X} \mathbf{X}^{T}$.
\end{proposition}
\proof
See Appendix~\ref{appendix:proof_cd-based scale optimization}.
\endproof
After refining the scales, we update the quantized weights (line~13 in \cref{algo:turboboa}), which yields the final output of the proposed \moa.

\section{Experiments}  \label{sec:experiments}

\subsection{Experimental Setup}  \label{subsec:setup}

We evaluate the performance of \moa \ using Llama models~\citep{touvron2023llama, touvron2023llama2}.
Following prior works~\citep{ashkboos2024quarot, liu2024spinquant, kim2024boa}, we use 128 sequences of length 2048 randomly sampled from the WikiText-2 (Wiki2) dataset~\cite{wiki} as calibration data for quantization.
As a performance metric, we use perplexity (PPL) on the Wiki2 and C4~\citep{c4} test sets and the average accuracy across eight zero-shot commonsense reasoning tasks.\footnote{ARC-challenge/easy~\citep{allenai:arc}, BoolQ~\citep{clark2019boolq}, HellaSwag~\citep{zellers2019hellaswag}, LAMBADA~\citep{paperno2016lambada}, OpenbookQA~\citep{mihaylov2018openbookqa}, PIQA~\citep{bisk2020piqa}, and WinoGrande~\citep{sakaguchi2021winogrande}}
All experiments were conducted using NVIDIA H100 GPUs (80~GB). 
While a single GPU was sufficient for most models, we utilized two GPUs for the 70B model to accommodate its larger memory requirements.

\paragraph{Hessian} 
We adopted Hessians derived in \boa \ (see \cref{tab:hessians}), since they are currently the most accurate closed-form Hessians available in the literature.
However, we emphasize that our main results in Propositions~\ref{prop:correction rule_multiple out-channels}-\ref{prop:cd-based scale optimization} are not specific to \boa's Hessians and can be applied to any Kronecker-structured Hessians $\mathbf{H} = \mathbf{H}_{in} \otimes \mathbf{H}_{out}$.
Consequently, our method can directly leverage more advanced Hessian formulations once they become available.

\paragraph{Joint Quantization Hyperparameter $N$} 
Our ablation study on the number $N$ of jointly quantized out-channels indicates that significant speedups are achievable up to $N=16$, beyond which further increases (\eg, $N=32$ or $64$) yield only marginal gains. 
To ensure stability, we conservatively set $N=16$ for all main experiments.
A detailed analysis is provided in \cref{subsec:ablation}. 

\paragraph{CD-based Scale Refinement} We set the number of CD iterations to one (\ie, $n_{iter}=1$ in~\cref{algo:cd}; see Appendix~\ref{appendix:proof_cd-based scale optimization}), as additional iterations yield only marginal improvements. The corresponding ablation study can be found in Appendix~\ref{appendix:ablation_cd iteration}.

\paragraph{Stabilization Coefficient $\alpha$} Following the implementation of GPTAQ~\citep{li2025gptaq}, we introduce a stabilization coefficient $\alpha$ to modulate the impact of the input deviation $\Delta \hspace{-.57mm} \mathbf{X}$ arising from the quantization errors of preceding layers (see line~5 in \cref{algo:turboboa}). This coefficient acts as a regularization parameter that prevents the compensation term from over-adjusting to accumulated distortions, which could otherwise lead to numerical instability. In our experiments, we evaluated $\alpha \in \{0.05, 0.125, 0.25\}$ and reported the best-performing result for each model.

\begin{table*}[!t]
    \renewcommand{\arraystretch}{1.0}
    \scriptsize
    \centering
    \caption{Ablation of multiple-row processing (INT2 quantization)}
    \vspace{-.3cm}
    \begin{threeparttable}
    \begin{tabular}{c c | c c | c c | c c | c c}
    \toprule
    \multirowcell{2}{\textbf{Method}} & \multicolumn{1}{c}{\multirowcell{2}{\textbf{$N$}}} & \multicolumn{2}{c}{\textbf{Llama3.2-1B}} & \multicolumn{2}{c}{\textbf{Llama3.2-3B}} & \multicolumn{2}{c}{\textbf{Llama3-8B}} & \multicolumn{2}{c}{\textbf{Llama3.1-70B}} \\
    & \multicolumn{1}{c}{} & Time (min) & \multicolumn{1}{c}{Wiki2 ($\downarrow$)} & Time (min) & \multicolumn{1}{c}{Wiki2 ($\downarrow$)} & Time (min) & Wiki2 ($\downarrow$) & Time (hr) & Wiki2 ($\downarrow$) \\
    \toprule
    \boa & 1 & 13.32 & 40.40 & 59.94 & 32.26 & 94.75 & 15.20 & 16.99 & 7.726 \\
    \midrule
    \multirowcell{5}{\boa \ + \textbf{F1}}
    & 4 & \textbf{6.255} & 41.09 & \textbf{22.68} & 32.21 & \textbf{39.46} & 15.27 & \textbf{7.683} & 7.721 \\
    & 8 & \textbf{5.002} & 41.53 & \textbf{16.01} & 31.66 & \textbf{30.55} & 15.30 & \textbf{6.274} & 7.714 \\
    & 16 & \textbf{4.363} & 41.85 & \textbf{12.70} & 31.99 & \textbf{25.30} & 15.41 & \textbf{5.636} & 7.758 \\
    & 32 & \textbf{3.985} & 41.75 & \textbf{11.01} & 32.15 & \textbf{22.95} & 15.22 & \textbf{5.060} & 7.746 \\
    & 64 & - & - & \textbf{10.29} & 32.31 & \textbf{21.56} & 15.44 & \textbf{4.885} & 7.774 \\
    \bottomrule
    \end{tabular}
    \begin{tablenotes}
        \item[*] Following \boa~\citep{kim2024boa}, QuaRot has been applied before quantizing weights. We note that \moa \ reduces to GPTQ under $N=64$ for Llama3.2-1B and $N=128$ for other models. 
    \end{tablenotes}
    \end{threeparttable}
    
    \label{tab:ablation_multiple row processing}
\end{table*}

\subsection{Ablation Studies} \label{subsec:ablation}

Recall that we incorporated three key features into the conventional \boa \ to accelerate the overall process and enhance the quantization performance.
To validate the effectiveness of each feature, we conduct ablation studies.

\paragraph{Speedup}
We first investigate the efficacy of the joint quantization of multiple $N$ out-channels (\textbf{F1}), introduced to mitigate the sequential bottleneck of \boa.
Specifically, we measure the processing time by varying $N \in \{ 4, 8, 16, 32, 64 \}$.
As expected, the processing time decreases significantly as more out-channels are quantized simultaneously (see \cref{tab:ablation_multiple row processing}); for example, when $N=16$, \moa \ achieves more than a three-fold speedup.
In particular, for the 70B model, this translates to a saving of 9$\sim$12 hours in absolute terms, demonstrating a substantial gain that becomes more impactful as model scale increases.

Intuitively, increasing $N$ reduces the degrees of freedom available for error compensation. 
However, our empirical results reveal that performance degradation remains negligible up to $N=64$, suggesting that the remaining out-channels provide sufficient capacity and the proposed update rule in~\cref{eq:update rule_multiple weights} effectively captures inter-channel correlations to compensate for joint quantization errors.
We leave a formal theoretical characterization of the error bounds with respect to $N$ as an interesting open question.
While this robustness to a large $N$ allows for aggressive parallelization, we observe that the speedup gain diminishes beyond $N=16$.
Therefore, we conservatively set $N=16$ for the remaining experiments to retain a higher margin of flexibility for error compensation.

\begin{table*}[!t]
    \renewcommand{\arraystretch}{1.0}
    \scriptsize
    \centering
    \caption{Ablation of features targeting performance enhancement (INT2 quantization)}
    \vspace{-.3cm}
    \begin{threeparttable}
    \begin{tabular}{c c c | c c c | c c c | c c c}
    \toprule
    \multirowcell{2}{\textbf{Method}} & \multirowcell{2}{\textbf{F2}} & \multicolumn{1}{c}{\multirowcell{2}{\textbf{F3}}} & \multicolumn{3}{c}{\textbf{Llama3.2-1B}} & \multicolumn{3}{c}{\textbf{Llama3.2-3B}} & \multicolumn{3}{c}{\textbf{Llama3-8B}} \\
    & & \multicolumn{1}{c}{} & Wiki2 ($\downarrow$) & C4 ($\downarrow$) & \multicolumn{1}{c}{Time} & Wiki2 ($\downarrow$) & C4 ($\downarrow$) & \multicolumn{1}{c}{Time} & Wiki2 ($\downarrow$) & C4 ($\downarrow$) & Time \\
    \toprule
    \boa & & & 40.40 & 104.9 & 13.32 & 32.26 & 79.17 & 59.94 & 15.20 & 36.95 & 94.75 \\
    \midrule
    \multirowcell{4}{\textbf{\moa} \\ ($N=16$)}
    & & & 41.85 & 108.1 & 4.363 & 31.99 & 80.09 & 12.70 & 15.41 & 38.96 & 25.30 \\
    & \checkmark & & 37.15 & 92.58 & 6.253 & 25.92 & 63.48 & 17.51 & 14.21 & 34.67 & 40.16 \\
    & & \checkmark & 39.45 & 107.3 & 4.426 & 31.12 & 73.57 & 12.84 & 15.01 & 36.40 & 25.39 \\
    & \checkmark & \checkmark & 33.33 & 85.55 & 6.263 & 24.10 & 54.20 & 17.71 & 13.54 & 32.99 & 40.20 \\
    \bottomrule
    \end{tabular}
    \begin{tablenotes}
        \item[*] Time in minutes. QuaRot has been applied before quantizing weights.
    \end{tablenotes}
    \end{threeparttable}
    
    \label{tab:ablation_performance improvement}
    \vspace{-.25cm}
\end{table*}

\paragraph{Performance Enhancement} 
We next examine the effectiveness of two features introduced to enhance the performance of \moa: error compensation for pre-quantized layers (\textbf{F2}) and adaptive grid computation with CD-based refinement (\textbf{F3}). 
As shown in \cref{tab:ablation_performance improvement}, incorporating either feature individually leads to consistent improvements. 
For example, on Llama3.2-1B, \textbf{F2} reduces PPL from 41.85 to 37.15 on Wiki2 and from 108.1 to 92.58 on C4, highlighting the benefit of mitigating error accumulation across layer depths. 
Similarly, \textbf{F3} improves alignment with the updated weight distribution, lowering PPL to 39.45 and 107.3 on Wiki2 and C4, respectively.
Notably, the combination of both features yields the best performance, demonstrating their complementary roles; \moa \ achieves PPLs of 33.33 on Wiki2 and 85.55 on C4 for Llama3.2-1B, representing substantial reductions over the baseline \boa. 
Consistent trends are observed in larger models, confirming that these enhancements generalize effectively across scales. 

Finally, we analyze the runtime overhead introduced by these features. 
While \textbf{F3} adds only a marginal cost (\eg, approximately one minute for Llama3-8B), the overhead of \textbf{F2} is more noticeable as it requires an additional forward pass of the FP model to compute the input deviation $\Delta \hspace{-.57mm} \mathbf{X}$. 
However, we emphasize that this is a fixed, one-time cost, as the FP activation $\widetilde{\mathbf{X}}$ is independent of the quantization process. 
Despite this overhead, \moa\ still completes the entire quantization process substantially faster than \boa, which confirms that the efficiency gains from reducing sequential operations via \textbf{F1} more than compensate for the additional computations required for accuracy enhancement.

\begin{table*}[t]
    \renewcommand{\arraystretch}{1.0}
    \fontsize{6pt}{7.2pt}\selectfont
    \centering
    \caption{Weight-only quantization performance on transformed Llama2 and Llama3 models}

    \vspace{-.1cm}

    \begin{subtable}{\textwidth}
    \centering
    \caption{PPL ($\downarrow$)}
    \vspace{-.2cm}
    \begin{tabular}{c | c | c | c c | c c | c c | c c | c c}
    \toprule
    \multicolumn{1}{c}{\multirowcell{2}{\textbf{Precision}}} & \multicolumn{1}{c}{\multirowcell{2}{\textbf{Transform}}} & \multicolumn{1}{c}{\multirowcell{2}{\textbf{Quantizer}}} & \multicolumn{2}{c}{\textbf{Llama3.2-1B}} & \multicolumn{2}{c}{\textbf{Llama3.2-3B}} & \multicolumn{2}{c}{\textbf{Llama3-8B}} & \multicolumn{2}{c}{\textbf{Llama2-7B}} & \multicolumn{2}{c}{\textbf{Llama2-13B}} \\
    \multicolumn{1}{c}{} & \multicolumn{1}{c}{} & \multicolumn{1}{c}{} & \multicolumn{1}{c}{Wiki2} & \multicolumn{1}{c}{C4} & \multicolumn{1}{c}{Wiki2} & \multicolumn{1}{c}{C4} & \multicolumn{1}{c}{Wiki2} & \multicolumn{1}{c}{C4} & \multicolumn{1}{c}{Wiki2} & \multicolumn{1}{c}{C4} & \multicolumn{1}{c}{Wiki2} & \multicolumn{1}{c}{C4} \\
    \toprule
    FP16 & \multicolumn{2}{c|}{Baseline} & 13.16 & 21.31 & 11.05 & 16.49 & 6.139 & 9.444 & 5.473 & 7.266 & 4.885 & 6.730 \\
    \midrule
    \multirowcell{9}{INT3}
    & OmniQuant$^{\dagger}$ & RTN & - & - & - & - & - & - & 6.640 & 9.383 & 5.593 & 7.989 \\
    & DuQuant & RTN & 2.7e4 & 1.8e4 & 15.18 & 22.31 & 10.78 & 17.90 & 6.226 & 8.645 & 5.414 & 7.598 \\
    \cmidrule{2-13}
    & \multirowcell{2}{SpinQuant} & RTN & 18.04 & 31.06 & 12.29 & 21.79 & 8.352 & 14.55 & 6.456 & 10.11 & 5.576 & 8.595 \\
    & & GPTQ & 16.21 & 27.60 & 12.87 & 20.47 & 7.438 & 12.75 & 6.001 & 8.619 & 5.299 & 7.682 \\
    \cmidrule{2-13}
    & \multirowcell{4}{QuaRot} & RTN & 98.24 & 139.0 & 89.54 & 101.1 & 38.64 & 51.43 & 129.2 & 111.9 & 48.06 & 48.79 \\
    & & GPTQ & 16.56 & 27.28 & 13.58 & 20.48 & 7.490 & 12.92 & 6.122 & 8.688 & 5.382 & 7.706 \\
    & & \boa & 15.73 & 26.15 & 12.97 & 19.96 & 7.145 & 12.25 & 5.874 & 8.268 & 5.202 & 7.436 \\
    & & \cellcolor{gray!15}\textbf{\moa}& \cellcolor{gray!15}\textbf{15.49} & \cellcolor{gray!15}\textbf{26.09} & \cellcolor{gray!15}\textbf{12.54} & \cellcolor{gray!15}\textbf{19.43} & \cellcolor{gray!15}\textbf{7.116} & \cellcolor{gray!15}\textbf{12.23} & \cellcolor{gray!15}\textbf{5.850} & \cellcolor{gray!15}\textbf{8.248} & \cellcolor{gray!15}\textbf{5.185} & \cellcolor{gray!15}\textbf{7.422} \\
    \midrule
    \multirowcell{9}{INT2}
    & OmniQuant$^{\dagger}$ & RTN & - & - & - & - & - & - & 21.85 & 39.34 & 12.92 & 19.99 \\
    & DuQuant & RTN & 9.3e3 & 1.6e4 & 770.9 & 905.7 & 2.6e4 & 1.8e5 & 46.27 & 69.02 & 10.40 & 15.35 \\
    \cmidrule{2-13}
    & \multirowcell{2}{SpinQuant} & RTN & 68.80 & 144.1 & 33.91 & 73.09 & 21.52 & 44.30 & 16.95 & 29.21 & 9.742 & 16.25 \\
    & & GPTQ & 48.64 & 127.1 & 34.65 & 92.42 & 15.86 & 39.11 & 15.43 & 30.30 & 9.652 & 19.35 \\
    \cmidrule{2-13}
    & \multirowcell{4}{QuaRot} & RTN & 2.6e5 & 2.5e5 & 2.3e4 & 1.1e4 & 3.5e5 & 3.6e5 & 1.1e4 & 1.1e4 & 7.9e3 & 6.2e3 \\
    & & GPTQ & 54.28 & 118.6 & 52.18 & 128.8 & 18.28 & 48.31 & 22.05 & 41.92 & 9.593 & 19.47 \\
    & & \boa & 40.86 & 107.9 & 33.40 & 79.21 & 15.24 & 36.82 & 10.42 & 19.17 & 8.237 & 14.66 \\
    & & \cellcolor{gray!15}\textbf{\moa} & \cellcolor{gray!15}\textbf{33.33} & \cellcolor{gray!15}\textbf{85.55} & \cellcolor{gray!15}\textbf{24.10} & \cellcolor{gray!15}\textbf{54.20} & \cellcolor{gray!15}\textbf{13.54} & \cellcolor{gray!15}\textbf{32.99} & \cellcolor{gray!15}\textbf{9.108} & \cellcolor{gray!15}\textbf{16.64} & \cellcolor{gray!15}\textbf{7.337} & \cellcolor{gray!15}\textbf{13.04} \\
    \bottomrule
    \end{tabular}
    \end{subtable}

    \vspace{.2cm}

    \begin{subtable}{\textwidth}
    \centering
    \caption{Zero-shot Accuracy ($\uparrow$)}
    \vspace{-.2cm}
    \begin{threeparttable}
    \begin{tabular}{c | c | c | c | c | c | c | c c c c c c}
    \toprule
    \multicolumn{1}{c}{\textbf{Precision}} & \multicolumn{1}{c}{\textbf{Transform}} & \multicolumn{1}{c}{\textbf{Quantizer}} & \multicolumn{1}{c}{\textbf{Llama3.2-1B}} & \multicolumn{1}{c}{\textbf{Llama3.2-3B}} & \multicolumn{1}{c}{\textbf{Llama3-8B}} & \multicolumn{1}{c}{\textbf{Llama2-7B}} & \multicolumn{1}{c}{\textbf{Llama2-13B}} \\
    \toprule
    FP16 & \multicolumn{2}{c|}{Baseline} & 56.82 & 63.01 & 70.34 & 67.28 & 69.83 \\
    \midrule
    \multirowcell{9}{INT3}
    & OmniQuant$^{\dagger}$ & RTN & - & - & - & 60.25 & 65.44 \\
    & DuQuant & RTN & 31.12 & 54.59 & 52.42 & 63.24 & 67.04 \\
    \cmidrule{2-8}
    & \multirowcell{2}{SpinQuant} & RTN & 48.65 & 56.69 & 64.32 & 60.40 & 65.32 \\
    & & GPTQ & 51.33 & 59.39 & 67.05 & 64.34 & 67.62 \\
    \cmidrule{2-8}
    & \multirowcell{4}{QuaRot} & RTN & 38.05 & 35.99 & 42.80 & 31.71 & 36.86 \\
    & & GPTQ & 51.13 & 57.89 & 66.67 & 63.72 & 67.79 \\
    & & \boa & 52.46 & 60.31 & 68.09 & 64.44 & 68.55 \\
    & & \cellcolor{gray!15}\textbf{\moa} & \cellcolor{gray!15}\textbf{53.32} & \cellcolor{gray!15}\textbf{61.26} & \cellcolor{gray!15}\textbf{68.57} & \cellcolor{gray!15}\textbf{65.21} & \cellcolor{gray!15}\textbf{69.07} \\
    \midrule
    \multirowcell{9}{INT2}
    & OmniQuant$^{\dagger}$ & RTN & - & - & - & 37.92 & 44.14 \\
    & DuQuant & RTN & 30.42 & 30.56 & 30.69 & 32.30 & 45.85 \\
    \cmidrule{2-8}
    & \multirowcell{2}{SpinQuant} & RTN & 35.97 & 37.94 & 42.25 & 38.95 & 47.41 \\
    & & GPTQ & 36.50 & 39.71 & 46.78 & 43.03 & 49.50 \\
    \cmidrule{2-8}
    & \multirowcell{4}{QuaRot} & RTN & 31.04 & 31.85 & 30.71 & 30.27 & 29.91 \\
    & & GPTQ & 36.43 & 39.17 & 45.02 & 38.98 & 49.51 \\
    & & \boa & 38.67 & 43.86 & 50.29 & 51.00 & 56.92 \\
    & & \cellcolor{gray!15}\textbf{\moa} & \cellcolor{gray!15}\textbf{40.31} & \cellcolor{gray!15}\textbf{45.85} & \cellcolor{gray!15}\textbf{52.59} & \cellcolor{gray!15}\textbf{53.27} & \cellcolor{gray!15}\textbf{59.69} \\
    \bottomrule
    \end{tabular}
    \begin{tablenotes}
        \item[${\dagger}$] The official code does not support models exploiting grouped query attention.
    \end{tablenotes}
    \end{threeparttable}
    \end{subtable}
    
    \label{tab:weight_only_quant_with_transform_main}
    \vspace{-.25cm}
\end{table*}

\begin{table*}[t]
    \renewcommand{\arraystretch}{1.0}
    \fontsize{6pt}{7.2pt}\selectfont
    \centering
    \caption{Weight-activation quantization performance on transformed Llama2 and Llama3 models}

    \vspace{-.1cm}
    
    \begin{subtable}{\textwidth}
    \centering
    \caption{PPL ($\downarrow$)}
    \vspace{-.2cm}
    \begin{tabular}{c | c | c | c c | c c | c c | c c | c c}
    \toprule
    \multicolumn{1}{c}{\multirowcell{2}{\textbf{Precision}}} & \multicolumn{1}{c}{\multirowcell{2}{\textbf{Transform}}} & \multicolumn{1}{c}{\multirowcell{2}{\textbf{Quantizer}}} & \multicolumn{2}{c}{\textbf{Llama3.2-1B}} & \multicolumn{2}{c}{\textbf{Llama3.2-3B}} & \multicolumn{2}{c}{\textbf{Llama3-8B}} & \multicolumn{2}{c}{\textbf{Llama2-7B}} & \multicolumn{2}{c}{\textbf{Llama2-13B}} \\
    \multicolumn{1}{c}{} & \multicolumn{1}{c}{} & \multicolumn{1}{c}{} & \multicolumn{1}{c}{Wiki2} & \multicolumn{1}{c}{C4} & \multicolumn{1}{c}{Wiki2} & \multicolumn{1}{c}{C4} & \multicolumn{1}{c}{Wiki2} & \multicolumn{1}{c}{C4} & \multicolumn{1}{c}{Wiki2} & \multicolumn{1}{c}{C4} & \multicolumn{1}{c}{Wiki2} & \multicolumn{1}{c}{C4} \\
    \toprule
    FP16 & \multicolumn{2}{c|}{Baseline} & 13.16 & 21.31 & 11.05 & 16.49 & 6.139 & 9.444 & 5.473 & 7.266 & 4.885 & 6.730 \\
    \midrule
    \multirowcell{9}{W2A4KV16}
    & OmniQuant$^{\dagger}$ & RTN & - & - & - & - & - & - & 2.3e3 & 3.2e3 & 2.8e3 & 4.4e3 \\
    & DuQuant & RTN & 1.0e4 & 1.5e4 & 1.2e3 & 1.6e3 & 4.4e4 & 2.5e5 & 375.0 & 514.3 & 13.25 & 20.12 \\
    \cmidrule{2-13}
    & \multirowcell{3}{SpinQuant} & GPTQ & 104.4 & 235.3 & 68.74 & 173.7 & 26.35 & 76.71 & 24.19 & 49.21 & 13.61 & 28.30 \\
    & & \boa & 59.95 & 136.5 & 34.24 & 110.0 & 17.31 & 48.04 & 11.27 & 19.86 & 8.652 & 15.33 \\
    & & \cellcolor{gray!15}\textbf{\moa} & \cellcolor{gray!15}\textbf{49.74} & \cellcolor{gray!15}\textbf{132.0} & \cellcolor{gray!15}\textbf{27.01} & \cellcolor{gray!15}\textbf{92.01} & \cellcolor{gray!15}\textbf{15.43} & \cellcolor{gray!15}\textbf{37.68} & \cellcolor{gray!15}\textbf{9.905} & \cellcolor{gray!15}\textbf{17.52} & \cellcolor{gray!15}\textbf
{7.862} & \cellcolor{gray!15}\textbf{13.64} \\
    \cmidrule{2-13}
    & \multirowcell{3}{OSTQuant} & GPTQ & 71.49 & 154.6 & 51.60 & 145.6 & 21.73 & 60.39 & 23.55 & 47.79 & 10.73 & 22.21 \\
    & & \boa & 44.90 & 107.7 & 29.90 & 74.04 & 15.16 & 37.49 & 10.07 & 18.22 & 7.894 & 13.96 \\
    & & \cellcolor{gray!15}\textbf{\moa} & \cellcolor{gray!15}\textbf{36.43} & \cellcolor{gray!15}\textbf{87.93} & \cellcolor{gray!15}\textbf{22.68} & \cellcolor{gray!15}\textbf{63.75} & \cellcolor{gray!15}\textbf{13.98} & \cellcolor{gray!15}\textbf{35.61} & \cellcolor{gray!15}\textbf{9.040} & \cellcolor{gray!15}\textbf{15.77} & \cellcolor{gray!15}\textbf{7.316} & \cellcolor{gray!15}\textbf{12.78} \\
    \midrule
    \multirowcell{9}{W2A4KV4}
    & OmniQuant$^{\dagger}$ & RTN & - & - & - & - & - & - & 1.0e5 & 1.9e5 & 3.8e3 & 5.4e3 \\
    & DuQuant & RTN & 8.5e3 & 1.5e4 & 1.7e3 & 2.5e3 & 4.1e4 & 2.5e5 & 465.9 & 753.3 & 16.35 & 24.83 \\
    \cmidrule{2-13}
    & \multirowcell{3}{SpinQuant} & GPTQ & 143.8 & 330.0 & 65.09 & 194.6 & 29.57 & 83.07 & 24.29 & 49.45 & 15.54 & 38.13 \\
    & & \boa & 77.05 & 167.0 & 37.12 & 120.0 & 18.23 & 48.52 & 11.80 & 20.97 & 8.974 & 15.96 \\
    & & \cellcolor{gray!15}\textbf{\moa} & \cellcolor{gray!15}\textbf{63.07} & \cellcolor{gray!15}\textbf{142.1} & \cellcolor{gray!15}\textbf{28.18} & \cellcolor{gray!15}\textbf{92.58} & \cellcolor{gray!15}\textbf{16.43} & \cellcolor{gray!15}\textbf{41.71} & \cellcolor{gray!15}\textbf{10.43} & \cellcolor{gray!15}\textbf{18.95} & \cellcolor{gray!15}\textbf{8.195} &\cellcolor{gray!15}\textbf{14.51} \\
    \cmidrule{2-13}
    & \multirowcell{3}{OSTQuant} & GPTQ & 80.61 & 206.1 & 60.37 & 214.6 & 23.87 & 68.52 & 21.53 & 42.89 & 11.32 & 22.47 \\
    & & \boa & 57.27 & 141.3 & 31.74 & 84.68 & 16.05 & 39.93 & 10.19 & 18.37 & 8.073 & 14.51 \\
    & & \cellcolor{gray!15}\textbf{\moa} & \cellcolor{gray!15}\textbf{46.10} & \cellcolor{gray!15}\textbf{111.7} & \cellcolor{gray!15}\textbf{24.53} & \cellcolor{gray!15}\textbf{72.72} & \cellcolor{gray!15}\textbf{14.51} & \cellcolor{gray!15}\textbf{38.12} & \cellcolor{gray!15}\textbf{9.142} & \cellcolor{gray!15}\textbf{16.59} & \cellcolor{gray!15}\textbf{7.508} & \cellcolor{gray!15}\textbf{13.25} \\
    \bottomrule
    \end{tabular}
    \end{subtable}

    \vspace{.2cm}

    \begin{subtable}{\textwidth}
    \centering
    \caption{Zero-shot Accuracy ($\uparrow$)}
    \vspace{-.2cm}
    \begin{threeparttable}
    \begin{tabular}{c | c | c | c | c | c | c | c c c c c c}
    \toprule
    \multicolumn{1}{c}{\textbf{Precision}} & \multicolumn{1}{c}{\textbf{Transform}} & \multicolumn{1}{c}{\textbf{Quantizer}} & \multicolumn{1}{c}{\textbf{Llama3.2-1B}} & \multicolumn{1}{c}{\textbf{Llama3.2-3B}} & \multicolumn{1}{c}{\textbf{Llama3-8B}} & \multicolumn{1}{c}{\textbf{Llama2-7B}} & \multicolumn{1}{c}{\textbf{Llama2-13B}} \\
    \toprule
    FP16 & \multicolumn{2}{c |}{Baseline} & 56.82 & 63.01 & 70.34 & 67.28 & 69.83 \\
    \midrule
    \multirowcell{9}{W2A4KV16}
    & OmniQuant$^{\dagger}$ & RTN & - & - & - & 30.63 & 30.19 \\
    & DuQuant & RTN & 30.58 & 30.47 & 30.77 & 30.45 & 41.72 \\
    \cmidrule{2-8}
    & \multirowcell{3}{SpinQuant} & GPTQ & 34.03 & 33.59 & 39.29 & 36.83 & 42.91 \\
    & & \boa & 36.56 & 39.56 & 44.53 & 48.25 & 54.35 \\
    & & \cellcolor{gray!15}\textbf{\moa} & \cellcolor{gray!15}\textbf{38.28} & \cellcolor{gray!15}\textbf{42.52} & \cellcolor{gray!15}\textbf{49.22} & \cellcolor{gray!15}\textbf{50.54} & \cellcolor{gray!15}\textbf{56.84} \\
    \cmidrule{2-8}
    & \multirowcell{3}{OSTQuant} & GPTQ & 35.28 & 35.42 & 40.92 & 38.59 & 45.08 \\
    & & \boa & 37.87 & 42.71 & 47.79 & 50.16 & 55.40 \\
    & & \cellcolor{gray!15}\textbf{\moa} & \cellcolor{gray!15}\textbf{39.47} & \cellcolor{gray!15}\textbf{45.80} & \cellcolor{gray!15}\textbf{50.49} & \cellcolor{gray!15}\textbf{52.14} & \cellcolor{gray!15}\textbf{58.77} \\
    \midrule
    \multirowcell{9}{W2A4KV4}
    & OmniQuant$^{\dagger}$ & RTN & - & - & - & 30.29 & 29.78 \\
    & DuQuant & RTN & 31.00 & 30.63 & 30.16 & 30.61 & 39.55 \\
    \cmidrule{2-8}
    & \multirowcell{3}{SpinQuant} & GPTQ & 33.59 & 33.31 & 37.26 & 37.54 & 40.02 \\
    & & \boa & 36.13 & 39.53 & 45.02 & 47.14 & 52.50 \\
    & & \cellcolor{gray!15}\textbf{\moa} & \cellcolor{gray!15}\textbf{37.28} & \cellcolor{gray!15}\textbf{42.44} & \cellcolor{gray!15}\textbf{47.75} & \cellcolor{gray!15}\textbf{49.89} & \cellcolor{gray!15}\textbf{55.86}  \\
    \cmidrule{2-8}
    & \multirowcell{3}{OSTQuant} & GPTQ & 33.90 & 35.32 & 41.70 & 36.82 & 46.54 \\
    & & \boa & 36.82 & 41.87 & 46.04 & 49.22 & 55.78 \\
    & & \cellcolor{gray!15}\textbf{\moa} & \cellcolor{gray!15}\textbf{39.35} & \cellcolor{gray!15}\textbf{44.08} & \cellcolor{gray!15}\textbf{49.78} & \cellcolor{gray!15}\textbf{51.44} & \cellcolor{gray!15}\textbf{58.23} \\
    \bottomrule
    \end{tabular}
    \begin{tablenotes}
        \item[${\dagger}$] The official code does not support models exploiting grouped query attention.
    \end{tablenotes}
    \end{threeparttable}
    \end{subtable}
    
    \label{tab:weight_act_quant_with_transform_main}
    \vspace{-.25cm}
\end{table*}

\subsection{Comparison with Prior Arts}

We now compare the performance of \moa \ against existing LLM quantization methods.
Our comparison includes \boa, which serves as the primary baseline, and transformation-based approaches that improve performance by suppressing outliers via scaling and/or rotation (\eg, OmniQuant~\citep{shao2023omniquant}, DuQuant~\citep{lin2024duquant}, QuaRot~\citep{ashkboos2024quarot}, SpinQuant~\citep{liu2024spinquant}, and OSTQuant~\citep{hu2025ostquant}); see Appendix~\ref{appendix:transformation-based PTQ methods} for the details of each method.

\paragraph{Weight-only Quantization}

We first evaluate the performance of weight-only quantization.
Following \boa~\citep{kim2024boa}, we integrate \moa \ with QuaRot, which requires no training and incurs no additional inference costs.
The complementarity with other transformation-based approaches (\eg, SpinQuant and OSTQuant) will be investigated in the weight-activation quantization setting (see \cref{tab:weight_act_quant_with_transform_main}).
For results without any transformation, please refer to Appendix~\ref{appendix:weight_only_no_transformation}, where we demonstrate the intrinsic effectiveness of the proposed error correction and grid selection.
Notably, in Appendix~\ref{appendix:comparison_with_gptaq}, we provide a direct comparison with GPTAQ~\citep{li2025gptaq} to highlight the importance of incorporating dependencies between out-channels in low-bit regimes.

\cref{tab:weight_only_quant_with_transform_main} summarizes the results under INT2 and INT3 quantization.
Overall, \boa \ and the proposed \moa \ outperform other methods because they explicitly account for cross-layer dependencies within the attention module during weight quantization.
In contrast, OmniQuant, SpinQuant-RTN, and SpinQuant-GPTQ consider cross-layer dependencies only when learning transformation matrices and rely on na\"{i}ve nearest rounding or GPTQ with layer-wise objectives, thereby failing to capture such dependencies during weight quantization.
As shown, \moa \ consistently achieves the best results.
For example, on 2-bit quantization of Llama3.2-1B, \moa \ improves Wiki2 PPL from 40.86 (\boa) to 33.33.
The benefits extend to zero-shot evaluation as well, where \moa \ achieves at least 2\%p accuracy gain over other methods across all model scales.
Notably, under 3-bit quantization, \moa \ nearly preserves the FP performance.
For instance, on Llama2-13B, \moa \ achieves 69.07\%, which is very close to the FP baseline of 69.83\%.

\paragraph{Weight-Activation Quantization}

We next evaluate the performance of weight-activation quantization.
Following prior works~\citep{ashkboos2024quarot, liu2024spinquant, kim2024boa}, we quantize input activations to all linear layers and KV caches using the Min-Max quantizer, where quantization parameters are dynamically computed for each token. 
For outlier suppression, we integrate GPTQ, \boa, and \moa \ with either SpinQuant or OSTQuant.
Unlike QuaRot, which relies on a fixed Hadamard matrix, SpinQuant and OSTQuant optimize rotation matrices by explicitly incorporating activation quantization effects during training~\citep{liu2024spinquant, hu2025ostquant}.

\cref{tab:weight_act_quant_with_transform_main} summarizes the results under W2A4KV4 and W2A4KV16 settings. 
Across both configurations and all model scales, \moa \ consistently outperforms \boa \ and other baselines. 
For example, with SpinQuant applied under W2A4KV4 on Llama3.2-1B, \moa \ reduces Wiki2 PPL from 77.05 (\boa) to 63.07.
When combined with OSTQuant under W2A4KV16 on Llama3.2-3B, \moa \ lowers C4 PPL from 74.04 (\boa) to 63.75, while GPTQ and DuQuant exhibit substantially higher PPLs.
Consistent gains are also observed for larger models such as Llama3-8B and Llama2-13B, confirming the scalability of the proposed approach.
Beyond PPL, \moa \ delivers clear improvements in zero-shot accuracy. 
On Llama3-8B under W2A4KV16, \moa \ with SpinQuant achieves 49.22\%, surpassing \boa \ by 5\%p. 
On Llama2-13B under W2A4KV4, \moa \ with SpinQuant attains 55.86\%, yielding an absolute gain of more than 3\%p over \boa \ and over 15\%p compared to GPTQ. 
These results demonstrate that \moa \ not only accelerates quantization but also achieves state-of-the-art performance in weight-activation quantization.

\section{Conclusion}  \label{sec:conclusion}

In this work, we proposed \moa, a backpropagation-free PTQ algorithm that addresses the key efficiency and accuracy bottlenecks of the conventional \boa. 
By quantizing multiple out-channels simultaneously, \moa \ significantly reduces sequential operations, accelerating the quantization process by more than three-fold. 
Furthermore, by extending error compensation to incorporate errors of previously quantized layers and adaptively determining quantization grids with a further CD-based refinement, \moa \ effectively mitigates error accumulation and misalignment, which could be critical in the low-bit regime.
Our experimental results demonstrate that \moa \ delivers substantial speedup over \boa \ while achieving superior accuracy, and when combined with transformation-based outlier suppression methods, it establishes new state-of-the-art results in both weight-only and weight-activation quantization. 
We believe \moa \ paves the way for broader deployment of LLMs on resource-constrained hardware, offering a practical balance between computational efficiency and model fidelity.

\newpage
\bibliography{iclr2026_conference}

@String(ICLR = {Int. Conf. Learn. Represent.})

@String(AAAI = {AAAI})

@String(ICLR  = {ICLR})

@inproceedings{nagel2020up,
  title={Up or down? {A}daptive rounding for post-training quantization},
  author={Nagel, Markus and Amjad, Rana Ali and Van Baalen, Mart and Louizos, Christos and Blankevoort, Tijmen},
  booktitle={International Conference on Machine Learning (ICML)},
  pages={7197--7206},
  year={2020}
}

@inproceedings{li2021brecq,
title={{BRECQ}: Pushing the Limit of Post-Training Quantization by Block Reconstruction},
author={Yuhang Li and Ruihao Gong and Xu Tan and Yang Yang and Peng Hu and Qi Zhang and Fengwei Yu and Wei Wang and Shi Gu},
booktitle={International Conference on Learning Representations (ICLR)},
year={2021}
}

@article{lecun1989optimal,
  title={Optimal brain damage},
  author={LeCun, Yann and Denker, John and Solla, Sara},
  journal={Advances in neural information processing systems},
  volume={2},
  year={1989}
}

@article{touvron2023llama,
  title={{LLaMA}: Open and efficient foundation language models},
  author={Touvron, Hugo and Lavril, Thibaut and Izacard, Gautier and Martinet, Xavier and Lachaux, Marie-Anne and Lacroix, Timoth{\'e}e and Rozi{\`e}re, Baptiste and Goyal, Naman and Hambro, Eric and Azhar, Faisal and others},
  journal={arXiv:2302.13971},
  year={2023}
}

@inproceedings{frantar2023optq,
  title={{OPTQ}: Accurate quantization for generative pre-trained {T}ransformers},
  author={Frantar, Elias and Ashkboos, Saleh and Hoefler, Torsten and Alistarh, Dan},
  booktitle={The Eleventh International Conference on Learning Representations},
  year={2023}
}

@inproceedings{xiao2023smoothquant,
  title={{SmoothQuant}: Accurate and efficient post-training quantization for large language models},
  author={Xiao, Guangxuan and Lin, Ji and Seznec, Mickael and Wu, Hao and Demouth, Julien and Han, Song},
  booktitle={International Conference on Machine Learning},
  pages={38087--38099},
  year={2023},
  organization={PMLR}
}

@inproceedings{jeon2023frustratingly,
  title={A frustratingly easy post-training quantization scheme for {LLMs}},
  author={Jeon, Yongkweon and Lee, Chungman and Park, Kyungphil and Kim, Ho-young},
  booktitle={Proceedings of the 2023 Conference on Empirical Methods in Natural Language Processing},
  pages={14446--14461},
  year={2023}
}

@article{shao2023omniquant,
  title={{OmniQuant}: Omnidirectionally calibrated quantization for large language models},
  author={Shao, Wenqi and Chen, Mengzhao and Zhang, Zhaoyang and Xu, Peng and Zhao, Lirui and Li, Zhiqian and Zhang, Kaipeng and Gao, Peng and Qiao, Yu and Luo, Ping},
  journal={arXiv:2308.13137},
  year={2023}
}

@article{c4,
  title={Exploring the limits of transfer learning with a unified text-to-text transformer},
  author={Raffel, Colin and Shazeer, Noam and Roberts, Adam and Lee, Katherine and Narang, Sharan and Matena, Michael and Zhou, Yanqi and Li, Wei and Liu, Peter J},
  journal={Journal of Machine Learning Research},
  volume={21},
  number={1},
  pages={5485--5551},
  year={2020},
  publisher={JMLRORG}
}

@article{wiki,
  title={Pointer sentinel mixture models},
  author={Merity, Stephen and Xiong, Caiming and Bradbury, James and Socher, Richard},
  journal={arXiv:1609.07843},
  year={2016}
}

@article{ashkboos2024quarot,
  title={{QuaRot}: Outlier-free 4-bit inference in rotated {LLMs}},
  author={Ashkboos, Saleh and Mohtashami, Amirkeivan and Croci, Maximilian L and Li, Bo and Cameron, Pashmina and Jaggi, Martin and Alistarh, Dan and Hoefler, Torsten and Hensman, James},
  journal={arXiv:2404.00456},
  year={2024}
}

@article{liu2024spinquant,
  title={{SpinQuant}: {LLM} quantization with learned rotations},
  author={Liu, Zechun and Zhao, Changsheng and Fedorov, Igor and Soran, Bilge and Choudhary, Dhruv and Krishnamoorthi, Raghuraman and Chandra, Vikas and Tian, Yuandong and Blankevoort, Tijmen},
  journal={arXiv:2405.16406},
  year={2024}
}

@inproceedings{lin2024duquant,
  title={{DuQuant}: Distributing outliers via dual transformation makes stronger quantized {LLMs}},
  author={Lin, Haokun and Xu, Haobo and Wu, Yichen and Cui, Jingzhi and Zhang, Yingtao and Mou, Linzhan and Song, Linqi and Sun, Zhenan and Wei, Ying},
  booktitle={The Thirty-eighth Annual Conference on Neural Information Processing Systems},
  year={2024}
}

@article{touvron2023llama2,
  title={Llama 2: Open foundation and fine-tuned chat models},
  author={Touvron, Hugo and Martin, Louis and Stone, Kevin and Albert, Peter and Almahairi, Amjad and Babaei, Yasmine and Bashlykov, Nikolay and Batra, Soumya and Bhargava, Prajjwal and Bhosale, Shruti and others},
  journal={arXiv:2307.09288},
  year={2023}
}

@article{allenai:arc,
      author    = {Peter Clark  and Isaac Cowhey and Oren Etzioni and Tushar Khot and
                    Ashish Sabharwal and Carissa Schoenick and Oyvind Tafjord},
      title     = {Think you have Solved Question Answering? {T}ry {ARC}, the {AI2} Reasoning Challenge},
      journal   = {arXiv:1803.05457v1},
      year      = {2018},
}

@article{clark2019boolq,
  title={{BoolQ}: Exploring the surprising difficulty of natural yes/no questions},
  author={Clark, Christopher and Lee, Kenton and Chang, Ming-Wei and Kwiatkowski, Tom and Collins, Michael and Toutanova, Kristina},
  journal={arXiv:1905.10044},
  year={2019}
}

@inproceedings{zellers2019hellaswag,
  title={{HellaSwag}: Can a Machine Really Finish Your Sentence?},
  author={Zellers, Rowan and Holtzman, Ari and Bisk, Yonatan and Farhadi, Ali and Choi, Yejin},
  booktitle={Proceedings of the 57th Annual Meeting of the Association for Computational Linguistics},
  pages={4791--4800},
  year={2019}
}

@inproceedings{bisk2020piqa,
  title={{PIQA}: Reasoning about physical commonsense in natural language},
  author={Bisk, Yonatan and Zellers, Rowan and Gao, Jianfeng and Choi, Yejin and others},
  booktitle={Proceedings of the AAAI conference on artificial intelligence},
  volume={34},
  number={05},
  pages={7432--7439},
  year={2020}
}

@article{mihaylov2018openbookqa,
  title={Can a suit of armor conduct electricity? {A} new dataset for open book question answering},
  author={Mihaylov, Todor and Clark, Peter and Khot, Tushar and Sabharwal, Ashish},
  journal={arXiv:1809.02789},
  year={2018}
}

@article{sakaguchi2021winogrande,
  title={{WinoGrande}: An adversarial winograd schema challenge at scale},
  author={Sakaguchi, Keisuke and Bras, Ronan Le and Bhagavatula, Chandra and Choi, Yejin},
  journal={Communications of the ACM},
  volume={64},
  number={9},
  pages={99--106},
  year={2021},
  publisher={ACM New York, NY, USA}
}

@article{paperno2016lambada,
  title={The {LAMBADA} dataset: Word prediction requiring a broad discourse context},
  author={Paperno, Denis and Kruszewski, Germ{\'a}n and Lazaridou, Angeliki and Pham, Quan Ngoc and Bernardi, Raffaella and Pezzelle, Sandro and Baroni, Marco and Boleda, Gemma and Fern{\'a}ndez, Raquel},
  journal={arXiv:1606.06031},
  year={2016}
}

@article{hu2025ostquant,
  title={{OSTQuant}: Refining large language model quantization with orthogonal and scaling transformations for better distribution fitting},
  author={Hu, Xing and Cheng, Yuan and Yang, Dawei and Xu, Zukang and Yuan, Zhihang and Yu, Jiangyong and Xu, Chen and Jiang, Zhe and Zhou, Sifan},
  journal={arXiv:2501.13987},
  year={2025}
}

@article{li2025gptaq,
  title={{GPTAQ}: Efficient Finetuning-Free Quantization for Asymmetric Calibration},
  author={Li, Yuhang and Yin, Ruokai and Lee, Donghyun and Xiao, Shiting and Panda, Priyadarshini},
  journal={arXiv:2504.02692},
  year={2025}
}

@article{aespa,
  title={Towards next-level post-training quantization of hyper-scale {T}ransformers},
  author={Kim, Junhan and Lee, Chungman and Cho, Eulrang and Park, Kyungphil and Kim, Ho-young and Kim, Joonyoung and Jeon, Yongkweon},
  journal={Advances in Neural Information Processing Systems},
  volume={37},
  pages={94292--94326},
  year={2024}
}

@inproceedings{kim2024boa,
title={{BoA}: Attention-aware Post-training Quantization without Backpropagation},
author={Junhan Kim and Ho-young Kim and Eulrang Cho and Chungman Lee and Joonyoung Kim and Yongkweon Jeon},
booktitle={Forty-second International Conference on Machine Learning (ICML)},
year={2025},
}
\bibliographystyle{iclr2026_conference}

\newpage
\appendix

\section{Use of LLMs}

In our work, LLMs were used solely to assist with paper writing, specifically for improving grammar, polishing phrasing, and enhancing readability.  
We have not used LLMs for developing research ideas, designing the methodology, conducting experiments, analyzing results, or drawing conclusions.

\section{Transformation-based PTQ Methods}
\label{appendix:transformation-based PTQ methods}

As noted, transformation-based methods aim to suppress outliers within weights or activations by applying a certain type of transformation, such as scaling, rotation, and permutation.
Transformation-based methods have often been used to improve the quantization robustness of models before conducting quantization~\citep{ashkboos2024quarot, liu2024spinquant, kim2024boa}.
If we denote a transformation matrix by $\mathbf{T}$, then the transformation in one layer can be expressed as
\begin{align}
    \mathbf{W} \mathbf{X}
        &= (\mathbf{W} \mathbf{T}) (\mathbf{T}^{-1} \mathbf{X}).
\end{align}
Under this formulation, the main goal of transformation-based methods is to construct a ``good" $\mathbf{T}$ that makes $\mathbf{W} \mathbf{T}$ and $\mathbf{T}^{-1} \mathbf{X}$ easier to be quantized.

Over the years, various transformation-methods have been proposed.
Each algorithm adopts a different strategy for constructing $\mathbf{T}$ to better suppress outliers and improve quantization robustness further.
For example, some methods adopt lightweight deterministic transformations, while others learn $\mathbf{T}$ through optimization guided by calibration data.
Below, we briefly summarize the contributions of each transformation-based method used in our comparison.

\textbf{OmniQuant} adopts a diagonal transformation matrix (\ie, $\mathbf{T} = \diag(\mathbf{c})$) to mitigate activation outliers that persist in several channels across all tokens~\citep{shao2023omniquant}. The scaling factor $\mathbf{c}$ is jointly optimized with the quantization parameters (scale and zero-point) of each layer via gradient-based training.
The learned scaling factor can be seamlessly merged into existing components (\eg, normalization layers), thereby incurring no additional inference overhead. When measuring performance, we activated both learnable equivalent transformation (LET) and learnable weight clipping (LWC) options.

\textbf{QuaRot/SpinQuant} adopt orthogonal (rotation) matrices $\mathbf{R}$  (\ie, $\mathbf{R} \mathbf{R}^{T} = \mathbf{R}^{T} \mathbf{R} = \mathbf{I}$), redistributing extremely large activation outliers that are present in few tokens~\citep{ashkboos2024quarot, liu2024spinquant}.
While QuaRot employs Hadamard matrices as the orthogonal transformation and thus requires no training, SpinQuant learns $\mathbf{R}$ guided by calibration data.
By applying the same orthogonal matrix across different Transformer layers, both methods can integrate $\mathbf{R}$ seamlessly into existing components, thereby incurring no inference overhead.

\textbf{DuQuant} integrates scaling $\diag(\mathbf{c})$, rotations $\mathbf{R}_{1}, \mathbf{R}_{2}$, and permutation $\mathbf{P}$ into a single transformation (\ie, $\mathbf{T} = \diag(\mathbf{c})\mathbf{R}_{1}\mathbf{P}\mathbf{R}_{2}$)~\citep{lin2024duquant}.
It provides an efficient backpropagation-free algorithm to compute the transformation parameters $\mathbf{c}, \mathbf{R}_{1}, \mathbf{R}_{2}, \mathbf{P}$ for each layer. Unlike QuaRot and SpinQuant, DuQuant learns distinct parameters for different Transformation blocks.
While this design incurs additional inference costs, the authors demonstrate through empirical timing measurements that the overhead remains manageable. When measuring performance, we activated the LWC option, which leads to the better result~\citep{lin2024duquant}.

\textbf{OSTQuant} combines scaling and rotation for transformation (\ie, $\mathbf{T} = \diag (\mathbf{c}) \mathbf{R}$)~\citep{hu2025ostquant}. In addition, it introduces a new metric, termed quantization space utilization rate (QSUR), to evaluate the quantizability of transformed data and provides a theoretical justification that the joint use of scaling and rotation improves QSUR. The scaling factors and rotation matrices are learned through gradient-based training and then fused into the original inference graph.

\newpage
\section{Update Rule for Error Compensation of Multiple Out-channels}  
\label{appendix:proof_update_formula_multiple_weights}

We solve the constrained optimization problem in~\cref{eq:error correction problem_multiple weights} by exploiting its Lagrangian:
\begin{align}
    L(\Delta \hspace{-.57mm} \mathbf{W}, \bm{\lambda}_{0}, \ldots \bm{\lambda}_{N-1})
        &= \| \mathbf{G} \Delta \hspace{-.57mm} \mathbf{W} \mathbf{X} \|_{F}^{2} + \sum_{i=0}^{N-1} (\mathbf{e}_{i}^{T} \Delta \hspace{-.57mm} \mathbf{W} - (\mathbf{Q}_{i, :} - \mathbf{W}_{i, :})) \bm{\lambda}_{i} \nonumber \\
        &= \trace (\mathbf{H}_{out} \Delta \hspace{-.57mm} \mathbf{W} \mathbf{H}_{in} \Delta \hspace{-.57mm} \mathbf{W}^{T}) + \sum_{i=0}^{N-1} (\mathbf{e}_{i}^{T} \Delta \hspace{-.57mm} \mathbf{W} + \mathbf{W}_{i, :} - \mathbf{Q}_{i, :}) \bm{\lambda}_{i}, \nonumber
\end{align}
where $\bm{\lambda}_{0}, \ldots, \bm{\lambda}_{N-1} \in \mathbb{R}^{d_{in} \times 1}$ are Lagrange multipliers.
Specifically, the update rule can be obtained by taking derivatives of the Lagrangian $L(\Delta \hspace{-.57mm} \mathbf{W}, \bm{\lambda}_{0}, \ldots \bm{\lambda}_{N-1})$ and then setting these derivatives to zero:
\begin{subequations}
\begin{align}
    \frac{\partial L}{\partial \Delta \hspace{-.57mm} \mathbf{W}}
    &= 2\mathbf{H}_{out} \Delta \hspace{-.57mm} \mathbf{W} \mathbf{H}_{in} + \sum_{i=0}^{N-1} \mathbf{e}_{i} \bm{\lambda}_{j}^{T} = \mathbf{0}_{d_{out} \times d_{in}}, \label{eq:lagrangian_1} \\
    \begin{bmatrix}
        (\partial{L} / \partial \bm{\lambda}_{0})^{T} \\
        \vdots \\
        (\partial{L} / \partial \bm{\lambda}_{N-1})^{T}
    \end{bmatrix}
    &= \begin{bmatrix}
        \mathbf{e}_{0}^{T} \Delta \hspace{-.57mm} \mathbf{W} + \mathbf{W}_{0, :} - \mathbf{Q}_{0, :} \\
        \vdots \\
        \mathbf{e}_{N-1}^{T} \Delta \hspace{-.57mm} \mathbf{W} + \mathbf{W}_{N-1, :} - \mathbf{Q}_{N-1, :}
    \end{bmatrix} 
    = [\Delta \hspace{-.57mm} \mathbf{W}]_{B, :} + (\mathbf{W}_{B, :} - \mathbf{Q}_{B, :})
    = \mathbf{0}_{N \times d_{in}} \label{eq:lagrangian_2}
\end{align}
\end{subequations}
As a result, the solution is attained when
\begin{subequations}
\begin{align}
    \Delta \hspace{-.57mm} \mathbf{W} 
    &= -\frac{1}{2} [\mathbf{H}_{out}^{-1}]_{:, B} \begin{bmatrix}
        \bm{\lambda}_{0}^{T} \\
        \vdots \\
        \bm{\lambda}_{N-1}^{T} 
    \end{bmatrix} \mathbf{H}_{in}^{-1}, \label{eq:lagrangian_3} \\
    [\Delta \hspace{-.57mm} \mathbf{W}]_{B, :}
    &= -(\mathbf{W}_{B, :} - \mathbf{Q}_{B, :}), \label{eq:lagrangian_4}
\end{align}
\end{subequations}
combining which yields 
\begin{align}
    \begin{bmatrix}
        \bm{\lambda}_{0}^{T} \\
        \vdots \\
        \bm{\lambda}_{N-1}^{T} 
    \end{bmatrix}
    = 2[\mathbf{H}_{out}^{-1}]_{B, B}^{-1} (\mathbf{W}_{B, :} - \mathbf{Q}_{B, :}) \mathbf{H}_{in}. \label{eq:lagrangian_5}
\end{align}
Finally, by combining~\cref{eq:lagrangian_3} and~\cref{eq:lagrangian_5}, we obtain the desired update rule in~\cref{eq:update rule_multiple weights}:
\begin{align} \label{eq:update rule_multiple weights_hessian}
    [\Delta \hspace{-.57mm} \mathbf{W}]_{N:, :}
        &= -[ \mathbf{H}_{out}^{-1} ]_{N:, B} [\mathbf{H}_{out}^{-1}]_{B, B}^{-1} (\mathbf{W}_{B, :} - \mathbf{Q}_{B, :}) \nonumber \\
        &\hspace{-.47mm}\overset{(a)}{=} -[ \mathbf{U}_{out}^{T} \mathbf{U}_{out} ]_{N:, B} [\mathbf{U}_{out}^{T} \mathbf{U}_{out}]_{B, B}^{-1} (\mathbf{W}_{B, :} - \mathbf{Q}_{B, :}) \nonumber \\
        &\hspace{-.47mm}\overset{(b)}{=} -[ \mathbf{U}_{out}^{T} ]_{N:, B} [\mathbf{U}_{out} ]_{B, B} \left ( [\mathbf{U}_{out}^{T}]_{B, B} [\mathbf{U}_{out}]_{B, B} \right )^{-1} (\mathbf{W}_{B, :} - \mathbf{Q}_{B, :}) \nonumber \\
        &=-[ \mathbf{U}_{out}^{T} ]_{N:, B} [\mathbf{U}_{out}^{T}]_{B, B}^{-1} (\mathbf{W}_{B, :} - \mathbf{Q}_{B, :}),
\end{align}
where (a) is because $\mathbf{U}_{out} = \chol(\mathbf{H}_{out}^{-1})^{T}$ and (b) is because $\mathbf{U}_{out}$ is upper triangular.

\newpage
\section{Update Rule Incorporating Quantization Errors of Earlier Transformer Blocks}  
\label{appendix:proof_update_formula_residual}

We note that the objective function~\cref{eq:error correction problem_multiple weights and residual} can be expressed as 
\begin{align}
    \| \mathbf{G} \Delta \hspace{-.57mm} \mathbf{W} \mathbf{X} + \mathbf{G}_{:, B} \mathbf{W}_{B, :} \Delta \hspace{-.57mm} \mathbf{X} \|_{F}^{2}
        &= \| \mathbf{G} \Delta \hspace{-0.57mm} \mathbf{W} \mathbf{X} \|_{F}^{2} + 2 \trace (\mathbf{G}_{:, B} \mathbf{W}_{B, :} \Delta \hspace{-.57mm} \mathbf{X} \mathbf{X}^{T} \Delta \hspace{-.57mm} \mathbf{W}^{T} \mathbf{G}^{T}) + c \nonumber \\
        &= \| \mathbf{G} \Delta \hspace{-0.57mm} \mathbf{W} \mathbf{X} \|_{F}^{2} + 2 \trace (\mathbf{G}^{T} \mathbf{G}_{:, B} \mathbf{W}_{B, :} \Delta \hspace{-.57mm} \mathbf{X} \mathbf{X}^{T} \Delta \hspace{-.57mm} \mathbf{W}^{T}) + c \nonumber \\
        &= \| \mathbf{G} \Delta \hspace{-0.57mm} \mathbf{W} \mathbf{X} \|_{F}^{2} + 2 \trace ([\mathbf{H}_{out}]_{:, B} \mathbf{W}_{B, :} \mathbf{R} \Delta \hspace{-.57mm} \mathbf{W}^{T}) + c, \nonumber
\end{align}
where $c = \| \mathbf{G}_{:, B} \mathbf{W}_{B, :} \Delta \hspace{-.57mm} \mathbf{X} \|_{F}^{2}$ is constant with respect to $\Delta \hspace{-.57mm} \mathbf{W}$ and the last equality holds because $\mathbf{H}_{out} = \mathbf{G}^{T} \mathbf{G}$ and $\mathbf{R} = \Delta \hspace{-.57mm} \mathbf{X} \mathbf{X}^{T}$.
Thus, the optimization problem in~\cref{eq:error correction problem_multiple weights and residual} is equivalent to
\begin{equation} \label{eq:formulation of moa_res_equivalent}
\begin{split}
    &\min_{\Delta \hspace{-.57mm} \mathbf{W}}~~\| \mathbf{G} \Delta \hspace{-.57mm} \mathbf{W} \mathbf{X} \|_{F}^{2} + 2 \trace ([\mathbf{H}_{out}]_{:, B} \mathbf{W}_{B, :} \mathbf{R} \Delta \hspace{-.57mm} \mathbf{W}^{T} ) \\
    &~\text{ s.t.}~~~~\mathbf{e}_{i}^{T} \Delta \hspace{-.57mm} \mathbf{W} = \mathbf{Q}_{i, :} - \mathbf{W}_{i, :}~(i \in B).
\end{split}
\end{equation}
Compared to the optimization problem in~\cref{eq:error correction problem_multiple weights} for the first layer, it involves an additional term in the objective whose derivative is 
\begin{align}
    \frac{\partial}{\partial \Delta \hspace{-.57mm} \mathbf{W}} \left ( 2 \trace ([\mathbf{H}_{out}]_{:, B} \mathbf{W}_{B, :}\mathbf{R} \Delta \hspace{-.57mm} \mathbf{W}^{T} ) \right )
        &= 2[\mathbf{H}_{out}]_{:, B} \mathbf{W}_{B, :} \mathbf{R}.
\end{align}
Using this together with~\cref{eq:lagrangian_1}, the solution is attained when 
\begin{align}
    \frac{\partial L}{\partial \Delta \hspace{-.57mm} \mathbf{W}}
    &= 2\mathbf{H}_{out} \Delta \hspace{-.57mm} \mathbf{W} \mathbf{H}_{in} + \sum_{i=0}^{N-1} \mathbf{e}_{i} \bm{\lambda}_{j}^{T} + 2[\mathbf{H}_{out}]_{:, B} \mathbf{W}_{B, :} \mathbf{R} = \mathbf{0}_{d_{out} \times d_{in}}, \nonumber
\end{align}
which is equivalent to
\begin{align}
    \Delta \hspace{-.57mm} \mathbf{W} 
    &= -\frac{1}{2} [\mathbf{H}_{out}^{-1}]_{:, B} \begin{bmatrix}
        \bm{\lambda}_{0}^{T} \\
        \vdots \\
        \bm{\lambda}_{N-1}^{T} 
    \end{bmatrix} \mathbf{H}_{in}^{-1} - \mathbf{I}_{:, B} \mathbf{W}_{B, :} \mathbf{R} \mathbf{H}_{in}^{-1}. \label{eq:lagrangian 1_residual}
\end{align}
Combining this with~\cref{eq:lagrangian_4} yields
\begin{align}
    [\Delta \hspace{-.57mm} \mathbf{W}]_{B, :}
    &= -\frac{1}{2} [\mathbf{H}_{out}^{-1}]_{B, B} \begin{bmatrix}
        \bm{\lambda}_{0}^{T} \\
        \vdots \\
        \bm{\lambda}_{N-1}^{T} 
    \end{bmatrix} \mathbf{H}_{in}^{-1} - \mathbf{W}_{B, :} \mathbf{R} \mathbf{H}_{in}^{-1}
    = -(\mathbf{W}_{B, :} - \mathbf{Q}_{B, :}), \nonumber 
\end{align}
which leads to 
\begin{align}
    \begin{bmatrix}
        \bm{\lambda}_{0}^{T} \\
        \vdots \\
        \bm{\lambda}_{N-1}^{T} 
    \end{bmatrix}
    = 2[\mathbf{H}_{out}^{-1}]_{B, B}^{-1} (\mathbf{W}_{B, :} - \mathbf{Q}_{B, :}) \mathbf{H}_{in} - 2[\mathbf{H}_{out}^{-1}]_{B, B}^{-1} \mathbf{W}_{B, :} \mathbf{R}. \label{eq:lagrangian 2_residual}
\end{align}
Finally, by combining~\cref{eq:lagrangian 1_residual} and~\cref{eq:lagrangian 2_residual}, we obtain the desired update rule in~\cref{eq:update rule_multiple weights and residual}:
\begin{align}
    [\Delta \hspace{-.57mm} \mathbf{W}]_{N:, :}
        &= -[ \mathbf{H}_{out}^{-1} ]_{N:, B} [\mathbf{H}_{out}^{-1}]_{B, B}^{-1} (\mathbf{W}_{B, :} - \mathbf{Q}_{B, :}) - (\mathbf{I}_{N:, B} - [\mathbf{H}_{out}^{-1}]_{N:, B} [\mathbf{H}_{out}^{-1}]_{B, B}^{-1}) \mathbf{W}_{B, :} \mathbf{R} \mathbf{H}_{in}^{-1} \nonumber \\
        &= -[ \mathbf{H}_{out}^{-1} ]_{N:, B} [\mathbf{H}_{out}^{-1}]_{B, B}^{-1} (\mathbf{W}_{B, :} - \mathbf{Q}_{B, :}) + [\mathbf{H}_{out}^{-1}]_{N:, B}[\mathbf{H}_{out}^{-1}]_{B, B}^{-1} \mathbf{W}_{B, :} \mathbf{R} \mathbf{H}_{in}^{-1} \nonumber \\
        &=-[ \mathbf{U}_{out}^{T} ]_{N:, B} [\mathbf{U}_{out}^{T}]_{B, B}^{-1} (\mathbf{W}_{B, :} - \mathbf{Q}_{B, :}) + [ \mathbf{U}_{out}^{T} ]_{N:, B} [\mathbf{U}_{out}^{T}]_{B, B}^{-1} \mathbf{W}_{B, :} \mathbf{R} \mathbf{H}_{in}^{-1} \nonumber,
\end{align}
where the last equality holds because $\mathbf{U}_{out} = \chol(\mathbf{H}_{out}^{-1})^{T}$ (see (a) and (b) in Appendix~\ref{appendix:proof_update_formula_multiple_weights}).

\newpage
\section{Attention-aware Scale Refinement via CD}  \label{appendix:proof_cd-based scale optimization}

Let $\mathbf{G} = [\mathbf{g}_{0} \cdots \mathbf{g}_{d_{out}-1}]$ and $\mathbf{W}_{int} = [\mathbf{w}_{int, 0} \cdots \mathbf{w}_{int, d_{out}-1}]^{T}$, then the attention reconstruction error in~\cref{eq:attention-wise recon loss_scale optimize} is expressed as
\begin{align}
    \mathcal{L}(\mathbf{s})
        &= \| \mathbf{G} \diag (\mathbf{s}) \mathbf{W}_{int} \mathbf{X} \|_{F}^{2} - 2 \langle \mathbf{G} \diag (\mathbf{s}) \mathbf{W}_{int} \mathbf{X}, \mathbf{G} \mathbf{W} \widetilde{\mathbf{X}} \rangle_{F} + c \nonumber \\
        &= \left \| \sum_{j=0}^{d_{out}-1} s_{j} \mathbf{g}_{j} \mathbf{w}_{int, j}^{T} \mathbf{X} \right \|_{F}^{2} - 2 \sum_{j=0}^{d_{out}-1} \langle s_{j} \mathbf{g}_{j} \mathbf{w}_{int, j}^{T} \mathbf{X}, \mathbf{G} \mathbf{W} \widetilde{\mathbf{X}} \rangle_{F} + c \nonumber \\
        &= \sum_{j, k=0}^{d_{out}-1} \trace(\mathbf{g}_{j} \mathbf{w}_{int, j}^{T} \mathbf{X} \mathbf{X}^{T} \mathbf{w}_{int, k} \mathbf{g}_{k}^{T} ) s_{j} s_{k} - 2 \sum_{j=0}^{d_{out}-1} \trace ( \mathbf{g}_{j} \mathbf{w}_{int, j}^{T} \mathbf{X} \widetilde{\mathbf{X}}^{T} \mathbf{W}^{T} \mathbf{G}^{T} )s_{j} + c, \nonumber
\end{align}
where $\langle \cdot, \cdot \rangle_{F}$ denotes the Frobenius inner product (\ie, $\langle \mathbf{A}, \mathbf{B} \rangle_{F} = \trace (\mathbf{A}\mathbf{B}^{T} )$) and $c$ is constant with respect to scales $\mathbf{s}$.
Using this together with $\mathbf{H}_{in} = \mathbf{X} \mathbf{X}^{T}$ and $\mathbf{X} \widetilde{\mathbf{X}}^{T} = \mathbf{X} (\mathbf{X} - \Delta \hspace{-.57mm} \mathbf{X})^{T} = (\mathbf{H}_{in} - \mathbf{R})^{T}$, we have
\begin{align*}
    \mathcal{L}(\mathbf{s})
        &= \sum_{j, k=0}^{d_{out}-1} \trace(\mathbf{g}_{j} \mathbf{w}_{int, j}^{T} \mathbf{H}_{in} \mathbf{w}_{int, k} \mathbf{g}_{k}^{T} ) s_{j}s_{k} - 2 \sum_{j=0}^{d_{out}-1} \trace ( \mathbf{g}_{j} \mathbf{w}_{int, j}^{T} (\mathbf{H}_{in} - \mathbf{R})^{T} \mathbf{W}^{T} \mathbf{G}^{T} ) s_{j} + c \\
        &\hspace{-5mm}= \sum_{j, k=0}^{d_{out}-1} (\mathbf{w}_{int, j}^{T} \mathbf{H}_{in} \mathbf{w}_{int, k} \mathbf{g}_{k}^{T} \mathbf{g}_{j}) s_{j}s_{k} - 2 \sum_{j=0}^{d_{out}-1} (\mathbf{w}_{int, j}^{T} (\mathbf{H}_{in} - \mathbf{R})^{T} \mathbf{W}^{T} \mathbf{G}^{T} \mathbf{g}_{j}) s_{j} + c \\
        &\hspace{-5mm}= \sum_{j, k=0}^{d_{out}-1} [\mathbf{W}_{int} \mathbf{H}_{in} \mathbf{W}_{int}^{T}]_{j, k} [\mathbf{H}_{out}]_{k, j} \cdot s_{j}s_{k} - 2 \sum_{j=0}^{d_{out}-1} [\mathbf{W}_{int} (\mathbf{H}_{in} - \mathbf{R})^{T} \mathbf{W}^{T} \mathbf{H}_{out}]_{j, j} \cdot s_{j} + c,
\end{align*}
where the last equality holds because $\mathbf{H}_{out} = \mathbf{G}^{T} \mathbf{G}$.
To minimize $\mathcal{L}(\mathbf{s})$, we adopt the CD algorithm, \ie, we iteratively update one scale at a time while keeping the others fixed.
Since the loss is quadratic in $s_{j}$, the update formula for $s_{j}$ can be obtained by setting $\partial \mathcal{L} / \partial s_{j} = 0$:
\begin{align}
    s_{j}^{*}
        &= \frac{[\mathbf{W}_{int} (\mathbf{H}_{in} - \mathbf{R})^{T} \mathbf{W}^{T} \mathbf{H}_{out}]_{j, j} - \sum_{k \neq j}[\mathbf{W}_{int} \mathbf{H}_{in} \mathbf{W}_{int}^{T}]_{j, k} [\mathbf{H}_{out}]_{k, j} s_{k}}{[\mathbf{W}_{int} \mathbf{H}_{in} \mathbf{W}_{int}^{T}]_{j, j} [\mathbf{H}_{out}]_{j, j}} \nonumber \\
        &= s_{j} + \frac{[\mathbf{W}_{int} (\mathbf{H}_{in} - \mathbf{R})^{T} \mathbf{W}^{T} \mathbf{H}_{out}]_{j, j} - \sum_{k}[\mathbf{W}_{int} \mathbf{H}_{in} \mathbf{W}_{int}^{T}]_{j, k} [\mathbf{H}_{out}]_{k, j} s_{k}}{[\mathbf{W}_{int} \mathbf{H}_{in} \mathbf{W}_{int}^{T}]_{j, j} [\mathbf{H}_{out}]_{j, j}} \nonumber \\
        &= s_{j} + \frac{[\mathbf{W}_{int} (\mathbf{H}_{in} - \mathbf{R})^{T} \mathbf{W}^{T} \mathbf{H}_{out}]_{j, j} - [\mathbf{W}_{int} \mathbf{H}_{in} \mathbf{Q}^{T} \mathbf{H}_{out}]_{j, j}}{[\mathbf{W}_{int} \mathbf{H}_{in} \mathbf{W}_{int}^{T}]_{j, j} [\mathbf{H}_{out}]_{j, j}} \nonumber \\
        &= s_{j} + \frac{[\mathbf{W}_{int} (\mathbf{H}_{in} (\mathbf{W} - \mathbf{Q})^{T} - \mathbf{R}^{T} \mathbf{W}^{T} )\mathbf{H}_{out}]_{j, j}}{[\mathbf{W}_{int} \mathbf{H}_{in} \mathbf{W}_{int}^{T}]_{j, j} [\mathbf{H}_{out}]_{j, j}}, \nonumber
\end{align}
which completes the proof.
In \cref{algo:cd}, we summarize the pseudocode for the CD-based scale refinement.

\begin{algorithm}[ht]
\begin{spacing}{1.0}
\footnotesize
\caption{Coordinate Descent-based Scale Refinement}
\label{algo:cd}
\renewcommand\algorithmicrequire{\textbf{Input}:}
\renewcommand\algorithmicensure{\textbf{Output}:}
\begin{algorithmic}[1]
\Require FP weights $\mathbf{W}$, integer weights $\mathbf{W}_{int}$, initial scales $\mathbf{s}$, Hessians $\mathbf{H}_{out}$ and $\mathbf{H}_{in}$, and deviation correlation $\mathbf{R}=\Delta \hspace{-.57mm} \mathbf{X}\mathbf{X}^{T}$
\Ensure refined scales $\mathbf{s}$

\State Initialize quantized weights: $\mathbf{Q} \leftarrow \diag (\mathbf{s}) \mathbf{W}_{int}$
\For{$\ell=0, \cdots, n_{iter}-1$}
\For{$j=0, \cdots, d_{out}-1$}
    \State Update scale for the $j$-th out-channel:
    \vspace{-1.5mm}
        \begin{equation*}
           s_j \leftarrow 
           s_{j} + 
           \frac{[\mathbf{W}_{int} (\mathbf{H}_{in} ( \mathbf{W} - \mathbf{Q} )^{T} - \mathbf{R}^{T} \mathbf{W}^{T} ) \mathbf{H}_{out}]_{j, j}}{[\mathbf{W}_{int} \mathbf{H}_{in} \mathbf{W}_{int}^{T}]_{j, j}[\mathbf{H}_{out}]_{j, j} }
        \end{equation*}
    \vspace{-2.5mm}
    \State Update quantized weights: $\mathbf{Q} \leftarrow \diag (\mathbf{s}) \mathbf{W}_{int}$
\EndFor
\EndFor
\end{algorithmic}
\end{spacing}
\end{algorithm}

\newpage
\section{Additional Experimental Results}

In this appendix, we present supplementary experimental results that were omitted from the main text due to page constraints. 
Specifically, we provide (i) weight-only quantization results without applying any transformations (\eg, scaling or rotation), (ii) a direct comparison with GPTAQ, (iii) weight-activation quantization results under higher weight bit-widths, and (iv) an ablation study on the number of CD iterations.

\subsection{Weight-only Quantization Performance without Transformation}  
\label{appendix:weight_only_no_transformation}

\cref{tab:weight_only_quant_no_transform} reports the weight-only quantization performance of GPTQ, \boa, and the proposed \moa \ without additional transformation. 
Across both 2-bit and 3-bit settings, \moa \ consistently outperforms both GPTQ and \boa.
For instance, on Llama3.2-1B (INT2), \moa \ significantly reduces Wiki2 PPL from 538.9 (GPTQ) and 312.2 (\boa) to 111.3, while simultaneously improving zero-shot accuracy by 2.5\%p.
Even under the INT3 setting, \moa \ achieves clear improvements over \boa, demonstrating that the proposed enhancements remain highly effective even in the absence of transformation-based outlier suppression.

\begin{table*}[ht]
    \renewcommand{\arraystretch}{1.0}
    \scriptsize
    \centering
    \caption{Weight-only quantization performance on Llama2 and Llama3 models}

    \vspace{-.1cm}

    \begin{subtable}{\textwidth}
    \centering
    \caption{Wiki2 PPL ($\downarrow$)}
    \vspace{-.2cm}
    \begin{tabular}{c c c c c c c c c c c c c}
    \toprule
    \textbf{Precision} & \textbf{Method} & \textbf{Llama3.2-1B} & \textbf{Llama3.2-3B} & \textbf{Llama3-8B} & \textbf{Llama2-7B} & \textbf{Llama2-13B} \\
    \toprule
    FP16 & Baseline & 13.16 & 11.05 & 6.139 & 5.473 & 4.885 \\
    \midrule
    \multirowcell{4.5}{INT3}
    & RTN & 1.9e3 & 882.6 & 129.1 & 342.4 & 227.2 \\
    & GPTQ & 112.0 & 46.14 & 8.226 & 6.719 & 9.790 \\
    & \boa & 26.43 & 13.64 & 7.782 & 6.007 & 5.833 \\
    \cmidrule{2-7}
    & \textbf{\moa} & \textbf{19.73} & \textbf{13.12} & \textbf{7.523} & \textbf{5.958} & \textbf{5.288} \\
    \midrule
    \multirowcell{4.5}{INT2}
    & RTN & 6.3e4 & 2.0e4 & 6.6e4 & 7.7e3 & 5.7e3 \\ 
    & GPTQ & 538.9 & 98.19 & 24.54 & 30.85 & 35.08 \\
    & \boa & 312.2 & 54.64 & 21.70 & 12.76 & 18.33 \\
    \cmidrule{2-7}
    & \textbf{\moa} & \textbf{111.3} & \textbf{33.42} & \textbf{17.83} & \textbf{9.781} & \textbf{13.09} \\
    \bottomrule
    \end{tabular}
    \end{subtable}

    \vspace{.2cm}

    \begin{subtable}{\textwidth}
    \centering
    \caption{C4 PPL ($\downarrow$)}
    \vspace{-.2cm}
    \begin{tabular}{c c c c c c c c c c c c c}
    \toprule
    \textbf{Precision} & \textbf{Method} & \textbf{Llama3.2-1B} & \textbf{Llama3.2-3B} & \textbf{Llama3-8B} & \textbf{Llama2-7B} & \textbf{Llama2-13B} \\
    \toprule
    FP16 & Baseline & 21.31 & 16.49 & 9.444 & 7.266 & 6.730 \\
    \midrule
    \multirowcell{4.5}{INT3}
    & RTN & 1.6e3 & 736.1 & 119.8 & 2.7e3 & 245.0 \\ 
    & GPTQ & 201.2 & 150.8 & 20.05 & 92.15 & 20.17 \\
    & \boa & 37.98 & 24.05 & 14.10 & 8.686 & 7.634 \\
    \cmidrule{2-7}
    & \textbf{\moa} & \textbf{36.43} & \textbf{23.79} & \textbf{13.59} & \textbf{8.554} & \textbf{7.587} \\
    \midrule
    \multirowcell{4.5}{INT2}
    & RTN & 4.6e4 & 1.1e4 & 8.2e4 & 8.2e3 & 4.8e3 \\
    & GPTQ & 1.2e3 & 413.8 & 214.3 & 321.1 & 97.52 \\
    & \boa & 571.9 & 214.0 & 92.69 & 26.42 & 28.36 \\
    \cmidrule{2-7}
    & \textbf{\moa} & \textbf{313.8} & \textbf{166.6} & \textbf{81.24} & \textbf{17.66} & \textbf{19.89} \\
    \bottomrule
    \end{tabular}
    \end{subtable}

    \vspace{.2cm}

    \begin{subtable}{\textwidth}
    \centering
    \caption{Zero-shot Accuracy ($\uparrow$)}
    \vspace{-.2cm}
    \begin{threeparttable}
    \begin{tabular}{c c c c c c c c c c c c c}
    \toprule
    \textbf{Precision} & \textbf{Method} & \textbf{Llama3.2-1B} & \textbf{Llama3.2-3B} & \textbf{Llama3-8B} & \textbf{Llama2-7B} & \textbf{Llama2-13B} \\
    \toprule
    FP16 & Baseline & 56.82 & 63.01 & 70.34 & 67.28 & 69.83 \\
    \midrule
    \multirowcell{4.5}{INT3}
    & RTN & 33.19 & 33.37 & 36.01 & 33.18 & 32.92 \\ 
    & GPTQ & 37.44 & 39.19 & 61.72 & 58.38 & 54.84 \\
    & \boa & 47.05 & 59.38 & 65.37 & 63.70 & 63.35 \\
    \cmidrule{2-7}
    & \textbf{\moa} & \textbf{47.46} & \textbf{59.67} & \textbf{67.07} & \textbf{64.17} & \textbf{67.24} \\
    \midrule
    \multirowcell{4.5}{INT2}
    & RTN & 31.08 & 30.99 & 32.79 & 30.19 & 30.12 \\
    & GPTQ & 30.48 & 34.39 & 36.01 & 42.50 & 39.08 \\
    & \boa & 31.33 & 38.53 & 42.16 & 45.81 & 45.06 \\
    \cmidrule{2-7}
    & \textbf{\moa} & \textbf{33.91} & \textbf{42.03} & \textbf{44.87} & \textbf{51.41} & \textbf{47.35} \\
    \bottomrule
    \end{tabular}
    \end{threeparttable}
    \end{subtable}
    
    \label{tab:weight_only_quant_no_transform}
    \vspace{-.25cm}
\end{table*}

\newpage

\subsection{Comparison with GPTAQ}
\label{appendix:comparison_with_gptaq}

To further validate the importance of incorporating inter-channel dependencies, we provide a direct comparison between GPTAQ~\citep{li2025gptaq} and \moa. 
While both algorithms aim to compensate for quantization errors from preceding layers, they differ fundamentally in their treatment of the out-channel-wise Hessian $\mathbf{H}_{out}$.
Specifically, while GPTAQ assumes $\mathbf{H}_{out}=\mathbf{I}$, thereby ignoring the correlations between out-channels, the proposed \moa \ explicitly incorporates the attention-aware Hessian $\mathbf{H}_{out}$ (see \cref{tab:hessians}) to capture these dependencies.

\cref{tab:comparison_with_gptaq} reports the performance on Llama3 models under the 2-bit weight-only quantization setting without additional transformations. 
Across all model scales, \moa \ consistently outperforms GPTAQ in both PPL and zero-shot accuracy. 
Notably, \moa \ achieves a 7.7\%p accuracy gain on Llama-3.2-3B and a 10.5\%p improvement on Llama-3-8B compared to GPTAQ. 
These results highlight that accounting for inter-channel dependencies is crucial for mitigating accuracy degradation in aggressive low-bit regimes.

\begin{table*}[ht]
    \renewcommand{\arraystretch}{1.0}
    \scriptsize
    \centering
    \caption{Evaluation on Llama3 models (INT2 quantization)}
    \label{tab:comparison_with_gptaq}
    \vspace{-.3cm}
    \begin{threeparttable}
    \begin{tabular}{c | c c | c c | c c}
    \toprule
    \multicolumn{1}{c}{\multirowcell{2}{\textbf{Method}}} & \multicolumn{2}{c}{\textbf{Llama3.2-1b}} & \multicolumn{2}{c}{\textbf{Llama3.2-3b}} & \multicolumn{2}{c}{\textbf{Llama3-8b}} \\
    \multicolumn{1}{c}{} & Wiki2 ($\downarrow$) & 0-shot ($\uparrow$) & Wiki2 ($\downarrow$) & 0-shot ($\uparrow$) & Wiki2 ($\downarrow$) & 0-shot ($\uparrow$) \\
    \toprule
    GPTAQ & 200.5 & 31.73 & 47.90 & 34.31 & 19.29 & 34.36 \\
    \midrule
    \textbf{\moa} & \textbf{111.3} & \textbf{33.91} & \textbf{33.42} & \textbf{42.03} & \textbf{17.83} & \textbf{44.87} \\
    \bottomrule
    \end{tabular}
    \end{threeparttable}
\end{table*}

\newpage

\subsection{Weight-Activation Quantization Performance under Higher Bit-widths}
\label{appendix:weight_activation}

We further report weight-activation quantization results under higher weight bit-widths in \cref{tab:weight_act_quant_with_transform_int4}.
In this table, results for OmniQuant are excluded because its official implementation does not support models utilizing grouped query attention.
As expected, the performance gap among different algorithms narrows in this regime, as 4-bit quantization preserves most of the original FP accuracy.
Nevertheless, \moa \ consistently provides robust improvements in almost all cases, confirming the effectiveness of our method even when quantization is less challenging.

\begin{table*}[ht]
    \renewcommand{\arraystretch}{1.0}
    \scriptsize
    \centering
    \caption{Weight-activation quantization performance on transformed Llama3 models}

    \vspace{-.1cm}
    
    \begin{subtable}{\textwidth}
    \centering
    \caption{PPL ($\downarrow$)}
    \vspace{-.2cm}
    \begin{tabular}{c | c | c | c c | c c | c c}
    \toprule
    \multicolumn{1}{c}{\multirowcell{2}{\textbf{Precision}}} & \multicolumn{1}{c}{\multirowcell{2}{\textbf{Transform}}} & \multicolumn{1}{c}{\multirowcell{2}{\textbf{Quantizer}}} & \multicolumn{2}{c}{\textbf{Llama3.2-1B}} & \multicolumn{2}{c}{\textbf{Llama3.2-3B}} & \multicolumn{2}{c}{\textbf{Llama3-8B}} \\
    \multicolumn{1}{c}{} & \multicolumn{1}{c}{} & \multicolumn{1}{c}{} & \multicolumn{1}{c}{Wiki2} & \multicolumn{1}{c}{C4} & \multicolumn{1}{c}{Wiki2} & \multicolumn{1}{c}{C4} & \multicolumn{1}{c}{Wiki2} & \multicolumn{1}{c}{C4} \\
    \toprule
    FP16 & \multicolumn{2}{c|}{Baseline} & 13.16 & 21.31 & 11.05 & 16.49 & 6.139 & 9.444 \\
    \midrule
    \multirowcell{8.5}{W4A4KV16}
    & DuQuant & RTN & 1.9e4 & 1.8e4 & 13.32 & 19.49 & 8.066 & 13.24 \\
    \cmidrule{2-9}
    & \multirowcell{3}{SpinQuant} & GPTQ & 16.68 & 26.87 & 11.87 & 19.47 & 7.636 & 12.59 \\
    & & \boa & 16.25 & 26.29 & 11.57 & \textbf{19.04} & 7.496 & 12.35 \\
    & & \cellcolor{gray!15}\textbf{\moa} & \cellcolor{gray!15}\textbf{16.09} & \cellcolor{gray!15}\textbf{26.12} & \cellcolor{gray!15}\textbf{11.55} & \cellcolor{gray!15}19.11 & \cellcolor{gray!15}\textbf{7.474} & \cellcolor{gray!15}\textbf{12.32} \\
    \cmidrule{2-9}
    & \multirowcell{3}{OSTQuant} & GPTQ & 16.02 & 25.26 & 11.88 & 18.60 & 7.349 & 12.04 \\
    & & \boa & 15.60 & 24.81 & 11.74 & 18.45 & 7.224 & 11.82 \\
    & & \cellcolor{gray!15}\textbf{\moa} & \cellcolor{gray!15}\textbf{15.53} & \cellcolor{gray!15}\textbf{24.68} & \cellcolor{gray!15}\textbf{11.69} & \cellcolor{gray!15}\textbf{18.31} & \cellcolor{gray!15}\textbf{7.213} & \cellcolor{gray!15}\textbf{11.78} \\
    \midrule
    \multirowcell{8.5}{W4A4KV4}
    & DuQuant & RTN & 1.7e4 & 1.4e4 & 13.84 & 20.52 & 8.402 & 13.59 \\
    \cmidrule{2-9}
    & \multirowcell{3}{SpinQuant} & GPTQ & 18.31 & 29.46 & 12.24 & 20.21 & 7.869 & 12.99 \\
    & & \boa & 17.83 & 28.65 & 11.98 & 19.84 & 7.705 & 12.75 \\
    & & \cellcolor{gray!15}\textbf{\moa} & \cellcolor{gray!15}\textbf{17.77} & \cellcolor{gray!15}\textbf{28.56} & \cellcolor{gray!15}\textbf{11.88} & \cellcolor{gray!15}\textbf{19.73} & \cellcolor{gray!15}\textbf{7.680} & \cellcolor{gray!15}\textbf{12.70} \\
    \cmidrule{2-9}
    & \multirowcell{3}{OSTQuant} & GPTQ & 17.29 & 28.19 & 12.64 & 20.00 & 7.540 & 12.42 \\
    & & \boa & 16.89 & 27.30 & 12.43 & 19.58 & 7.428 & 12.22 \\
    & & \cellcolor{gray!15}\textbf{\moa} & \cellcolor{gray!15}\textbf{16.86} & \cellcolor{gray!15}\textbf{27.10} & \cellcolor{gray!15}\textbf{12.39} & \cellcolor{gray!15}\textbf{19.45} & \cellcolor{gray!15}\textbf{7.416} & \cellcolor{gray!15}\textbf{12.20} \\
    \bottomrule
    \end{tabular}
    \end{subtable}

    \vspace{.2cm}

    \begin{subtable}{\textwidth}
    \centering
    \caption{Zero-shot Accuracy ($\uparrow$)}
    \vspace{-.2cm}
    \begin{threeparttable}
    \begin{tabular}{c | c | c | c | c | c c c}
    \toprule
    \multicolumn{1}{c}{\textbf{Precision}} & \multicolumn{1}{c}{\textbf{Transform}} & \multicolumn{1}{c}{\textbf{Quantizer}} & \multicolumn{1}{c}{\textbf{Llama3.2-1B}} & \multicolumn{1}{c}{\textbf{Llama3.2-3B}} & \multicolumn{1}{c}{\textbf{Llama3-8B}} \\
    \toprule
    FP16 & \multicolumn{2}{c |}{Baseline} & 56.82 & 63.01 & 70.34 \\
    \midrule
    \multirowcell{8.5}{W4A4KV16}
    & DuQuant & RTN & 30.33 & 57.93 & 63.15 \\
    \cmidrule{2-6}
    & \multirowcell{3}{SpinQuant} & GPTQ & 50.89 & 58.71 & 64.79 \\
    & & \boa & 51.76 & 59.17 & 65.31 \\
    & & \cellcolor{gray!15}\textbf{\moa} & \cellcolor{gray!15}\textbf{52.32} & \cellcolor{gray!15}\textbf{59.42} & \cellcolor{gray!15}\textbf{66.15} \\
    \cmidrule{2-6}
    & \multirowcell{3}{OSTQuant} & GPTQ & 52.48 & 60.16 & 66.66 \\
    & & \boa & 53.24 & 60.94 & 67.43 \\
    & & \cellcolor{gray!15}\textbf{\moa} & \cellcolor{gray!15}\textbf{53.67} & \cellcolor{gray!15}\textbf{61.65} & \cellcolor{gray!15}\textbf{67.88} \\
    \midrule
    \multirowcell{8.5}{W4A4KV4}
    & DuQuant & RTN & 30.71 & 56.53 & 62.76 \\
    \cmidrule{2-6}
    & \multirowcell{3}{SpinQuant} & GPTQ & 48.86 & 57.54 & 64.05 \\
    & & \boa & 50.41 & \textbf{58.90} & 65.03 \\
    & & \cellcolor{gray!15}\textbf{\moa} & \cellcolor{gray!15}\textbf{50.73} & \cellcolor{gray!15}58.77 & \cellcolor{gray!15}\textbf{65.64} \\
    \cmidrule{2-6}
    & \multirowcell{3}{OSTQuant} & GPTQ & 50.44 & 59.34 & 65.25 \\
    & & \boa & 50.94 & 59.66 & 66.47 \\
    & & \cellcolor{gray!15}\textbf{\moa} & \cellcolor{gray!15}\textbf{51.54} & \cellcolor{gray!15}\textbf{59.86} & \cellcolor{gray!15}\textbf{66.73} \\
    \bottomrule
    \end{tabular}
    \end{threeparttable}
    \end{subtable}
    
    \label{tab:weight_act_quant_with_transform_int4}
    \vspace{-.25cm}
\end{table*}

\newpage

\subsection{Ablation on the number of CD iterations}
\label{appendix:ablation_cd iteration}

In this subsection, we investigate the impact of the number $n_{iter}$ of CD iterations (see~\cref{algo:cd}) on quantization quality.
We focus on the attention reconstruction loss $\| \mathbf{G} \Delta \mathbf{W} \mathbf{X} \|_{F}^{2}$ measured at the first Transformer block to avoid confounding effects from previous blocks.
The results in~\cref{tab:ablation_cd} indicate that the first CD iteration accounts for nearly all the reduction in loss, with additional iterations yielding diminishing returns.
Accordingly, the end-to-end PPL performance remains virtually unchanged between 1 and 2 iterations.
To maintain optimal computational efficiency, we set the CD iteration count to 1 for all main experiments.

\begin{table*}[ht]
    \renewcommand{\arraystretch}{1.0}
    \scriptsize
    \centering
    \caption{Ablation on the number of CD iterations}
    \vspace{-.3cm}
    \begin{threeparttable}
    \begin{tabular}{c c c c c c}
    \toprule
    \multicolumn{1}{c}{Model} & \multicolumn{1}{c}{$n_{iter}$} & Loss (Query) & Loss (Key) & Wiki2 ($\downarrow$) & C4 ($\downarrow$) \\
    \toprule
    \multirowcell{3}{Llama3.2-1B} 
    & 0 & 317.6 & 66.97 & 37.15 & 92.58 \\
    & 1 & 315.9 & 66.68 & 33.33 & 85.55 \\
    & 2 & 315.8 & 66.67 & 32.28 & 88.03 \\
    \midrule
    \multirowcell{3}{Llama3.2-3B} 
    & 0 & 170.1 & 70.72 & 25.92 & 63.48 \\
    & 1 & 168.9 & 70.25 & 24.10 & 54.20 \\
    & 2 & 168.7 & 70.16 & 24.07 & 54.53 \\
    \midrule
    \multirowcell{3}{Llama3-8B} 
    & 0 & 126.1 & 43.36 & 14.21 & 34.67 \\
    & 1 & 125.6 & 43.20 & 13.54 & 32.99 \\
    & 2 & 125.5 & 43.17 & 13.46 & 33.23 \\
    \bottomrule
    \end{tabular}
    \end{threeparttable}
    \label{tab:ablation_cd}
\end{table*}

\newpage
\section{Pseudocode for GPTAQ} \label{appendix:pseudocode_gptaq}

In this appendix, we provide the pseudocode of the conventional GPTAQ~\citep{li2025gptaq}, which is omitted in the main manuscript due to the page limitation.

\begin{algorithm*}[!htb]
\begin{spacing}{1.05}
\caption{GPTAQ} 
\footnotesize
\label{algo:gptaq}
\renewcommand\algorithmicrequire{\textbf{Input}:}
\renewcommand\algorithmicensure{\textbf{Output}:}
\begin{algorithmic}[1]
\Require weights $\mathbf{W}$, Hessian information $\mathbf{U}_{in}$, deviation correlation $\mathbf{R}=\Delta \hspace{-.57mm} \mathbf{X} \mathbf{X}^{T}$, and scale $\mathbf{s}$
    \State Initialize quantized and integer weights: $\mathbf{Q}, \mathbf{W}_{int} \leftarrow \mathbf{0}_{d_{out} \times d_{in}}$
    \State Compute $\mathbf{P} = \left ( \mathbf{R} \mathbf{U}_{in}^{T} \odot \mathbf{M}_{\mathbf{U}} \right ) \mathbf{U}_{in}$ ($\mathbf{M}_{\mathbf{U}}$: strictly upper triangular masking matrix with ones above the diagonal)
    \For{$j=0, \cdots, d_{in} - 1$}
    \State Quantize the $j$-th in-channel:
    \vspace{-2.5mm}
    \begin{align*}
        &[\mathbf{W}_{int}]_{:, j} \leftarrow \text{clamp} \left ( \left \lfloor \diag (\mathbf{s})^{-1} \mathbf{W}_{:, j} \right \rceil, 0, 2^{b} - 1 \right ) \\
        &\mathbf{Q}_{:, j} \leftarrow \diag(\mathbf{s}) [\mathbf{W}_{int}]_{:, j}
    \end{align*}
    \vspace{-5.5mm}
    \State Estimate quantization error: $\mathbf{E}_{:, j} \leftarrow (\mathbf{W}_{:, j} - \mathbf{Q}_{:, j}) / [\mathbf{U}_{in}]_{j, j}$
    \State Update remaining in-channels: 
    \vspace{-1.5mm}
    \begin{align*}
        \mathbf{W}_{:, j:} \leftarrow \mathbf{W}_{:, j:} - \frac{\mathbf{W}_{:, j} - \mathbf{Q}_{:, j}}{[\mathbf{U}_{in}]_{j, j}} [\mathbf{U}_{in}]_{j, j:} - \mathbf{W}_{:, j} \mathbf{P}_{j, j:}
    \end{align*}
    \vspace{-2.5mm}
    \EndFor
\Ensure quantized weights $\mathbf{Q}$, integer weights $\mathbf{W}_{int}$
\end{algorithmic}
\end{spacing}
\end{algorithm*}

\end{document}